\documentclass{article}
\usepackage{rescom_preprint,times}
\usepackage{float,capt-of}
\usepackage{booktabs,multirow,amsmath,amssymb,graphicx}
\usepackage[table]{xcolor}
\usepackage{placeins,flafter,amsthm,mathtools,algorithm,algpseudocode}
\usepackage{hyperref,url,xurl}
\hypersetup{pdftitle={ResCom: Activation Residual Compensation for Large Language Model Quantization},
  pdfauthor={Ning Yang, Bohan Cui, Zhuowen Liu, Yuting Wang}}
\newtheorem{proposition}{Proposition}
\newcommand{\method}{\textsc{ResCom}}
\newcommand{\tbd}{\textit{--}}
\algrenewcommand\algorithmicrequire{\textbf{Input:}}
\algrenewcommand\algorithmicensure{\textbf{Output:}}
\renewcommand{\arraystretch}{1.05}
\title{\method{}: Activation Residual Compensation\\for Large Language Model Quantization}
\author{%
\makebox[0.44\textwidth][l]{Ning Yang}\\
{\normalfont Institute of Automation,}\\
{\normalfont Chinese Academy of Sciences}
\And
\makebox[0.44\textwidth][l]{Bohan Cui\setcounter{footnote}{1}\thanks{Corresponding author.}}\\
{\normalfont Fudan University}\\
{\normalfont\texttt{bhcui24@m.fudan.edu.cn}}\\
{\normalfont\texttt{cuibohan010@gmail.com}}
\AND
\makebox[0.44\textwidth][l]{Zhuowen Liu}\\
{\normalfont The Chinese University}\\
{\normalfont of Hong Kong}
\And
\makebox[0.44\textwidth][l]{Yuting Wang}\\
{\normalfont University of Chinese Academy}\\
{\normalfont of Sciences}%
}

\begin{document}
\raggedbottom
\maketitle
\begin{abstract}
Low-bit weight--activation quantization reduces the memory and computation
needed to deploy large language models, but activation rounding can still
degrade accuracy after calibration. Reconstruction from rounded inputs cannot
distinguish activations within the same quantization bin. We propose ResCom, a
low-rank correction that uses the activation rounding residual to compensate
for the resulting layer-output error. An output-aware calibration procedure
selects residual directions by their contribution to output reconstruction
and computes the correction in closed form, without gradient training or
updates to the calibrated weights. On Llama-3-8B with 4-bit weights and
activations, rank-32 attention corrections reduce WikiText-2 perplexity from
7.708 to 7.147, improve MMLU by 2.89 percentage points, and raise five-shot
GSM8K accuracy from 23.96 to 36.16 percent. Evaluations on models from 1B
to 70B parameters and on held-out C4 test the method beyond this setting.
The evaluated FP32 implementation adds 2.25 percent to 4K prefill
time; a separate compact-storage implementation retains a 27.7 percent
peak-memory reduction relative to FP16.
\end{abstract}

\section{Introduction}
\label{sec:intro}
\begingroup
\setlength{\parskip}{3pt}

Large language models (LLMs) continue to improve in capability, but their memory and computational requirements make inference costly~\citep{pope2023efficiently}. Quantization reduces these costs by representing model values with fewer bits. Post-training quantization (PTQ) applies this conversion to pretrained models using limited calibration data, avoiding a full quantization-aware training procedure~\citep{nagel2021whitepaper}. Since linear layers multiply learned weights by input-dependent activations, quantizing both operands enables low-bit matrix multiplication on compatible hardware. With four-bit weights and activations (W4A4), however, few numerical levels represent activation ranges that can be enlarged by outliers~\citep{dettmers2022llmint8}. For a uniform quantizer, a wider range spreads these levels farther apart, so wider intervals of activation values share each rounded value. The layer consequently loses input distinctions that may affect its output.

To reduce this distortion, activation-shaping methods adjust distributions before quantization through scaling or transformations~\citep{xiao2023smoothquant,ashkboos2024quarot}. The resulting low-bit computation can still incur output error, which error-reconstruction methods compensate with compact branches~\citep{zhang2024lqer,zhang2025qera}. For W4A4, SERQ implements reconstruction with a single compensation matrix operating on quantized activations~\citep{park2026serq}. LRC jointly optimizes quantized weights and a correction, whereas ARHQ uses residual statistics to determine a weight split before quantizing the remaining weights~\citep{scetbon2024lrc,wang2026arhq}. Their correction procedures therefore participate in constructing the quantized model. Once a model has been calibrated, preserving its weights and input transformations allows those calibration choices to be reused while fitting an additional correction.

Different activations can produce different layer outputs even when they round to the same vector. A correction conditioned only on that vector, as in SERQ, gives these inputs the same adjustment, leaving their distinct output contributions unresolved. This motivates a correction that accesses the discarded information while retaining the already calibrated computation. How can such a correction recover layer outputs under a limited computation budget?

To address this question, we propose \method{}, which adds a residual-conditioned low-rank correction to an already calibrated quantized model, referred to as the \emph{host model} (Figure~\ref{fig:flow}). \method{} fits the correction to layer-output error while retaining the existing weights, quantization settings, and reconstruction path. We evaluate this post-calibration attachment beyond SERQ using Llama-3-8B checkpoints from PrefixQuant and FlatQuant~\citep{chen2025prefixquant,sun2025flatquant}.

The branch receives the activation rounding residual, the difference between an activation before quantization and the value reconstructed from its quantized representation. This residual records what rounding discards and is formed at the existing input quantizer. In our SERQ instantiation, the frozen linear map includes both the dequantized main-path weights and the existing reconstruction term. Because both inputs are evaluated under the same frozen map, its response to the residual exactly equals the output difference caused by activation rounding. This activation-induced output difference defines our SERQ calibration target. Using the residual isolates the input component absent from the main path and directly matches this target's linear dependence on the discarded information. Evaluating the frozen map on the residual would recover this difference directly, but requires a dense high-precision multiplication. To bound the additional cost, \method{} uses a low-rank approximation implemented by two smaller matrix multiplications, trading recovery capacity for computation and storage. Under this budget, residual magnitude alone overlooks output amplification, while weights alone ignore which residual variations occur; a small residual can still have a large output effect. \method{} therefore fits a regularized reduced-rank regression to output error, using a ridge term to stabilize fitting from limited calibration data and obtaining a closed-form solution without gradient training.

\begin{figure}[!htbp]
\centering
\includegraphics[width=0.9\textwidth]{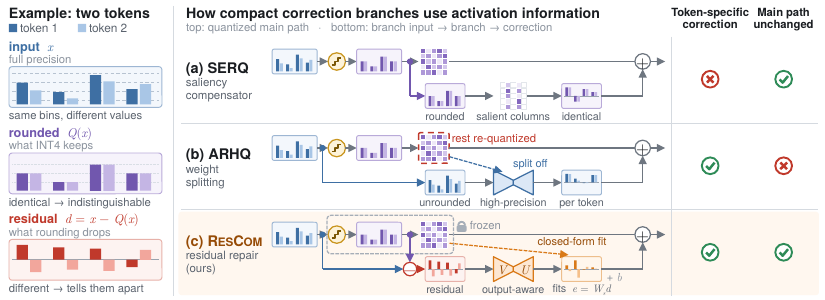}
\small
\caption{Three reconstruction paths. \textbf{(a)} SERQ applies a selected-channel
compensator to rounded activations. \textbf{(b)} ARHQ splits weights before
quantization and uses unrounded activations. \textbf{(c)} ResCom instead feeds
the activation rounding residual to an output-aware low-rank branch fitted after
calibration. 
Appendices~\ref{app:priorart}--\ref{app:repro} formalize these differences.}
\label{fig:flow}
\vspace{-4pt}
\end{figure}

Our contributions are as follows:
\begin{itemize}
\setlength{\itemsep}{1pt}
\setlength{\topsep}{2pt}
\item \textbf{Residual-conditioned correction after calibration.} We propose \method{}, a compact correction that uses activation rounding residuals to recover output-relevant information while preserving the calibrated weights, quantization settings, and existing reconstruction path.
\item \textbf{Output-aware residual fitting.} We formulate residual-to-output calibration as regularized reduced-rank regression, allocating a prescribed rank budget according to layer-output reconstruction error. The resulting closed-form fit calibrates the added branch without updating the host model.
\item \textbf{Accuracy recovery across model scales.} Rank-32 attention correction on query, key, value, and output projections lowers WikiText-2 perplexity on all nine W4A4 SERQ checkpoints (1B--70B). On Llama-3-8B, perplexity falls from 7.708 to 7.147, MMLU improves by 2.89 percentage points, and five-shot GSM8K strict accuracy rises from 23.96\% to 36.16\%.
\end{itemize}

\section{Related Work}
\label{sec:related}

\paragraph{Low-bit quantization and activation shaping.}
Activation shaping improves low-bit accuracy by limiting how outliers enlarge the range represented by a quantizer. \citet{xiao2023smoothquant} shift quantization difficulty from activations to weights through channel scaling in SmoothQuant. Suppressing activation outliers can enlarge weight ranges, making the balance between the two errors important. \citet{shao2024omniquant} optimize this balance by jointly learning equivalent transformations and weight-clipping parameters in OmniQuant. Scaling changes each channel's magnitude while keeping its values within that channel. To spread outlier values across channels, \citet{ashkboos2024quarot} use output-preserving Hadamard rotations in QuaRot. Rotations mix channels without independently rescaling them; \citet{sun2025flatquant} combine mixing and rescaling through learned affine transformations in FlatQuant. Beyond these feature transformations, \citet{chen2025prefixquant} isolate token-associated outliers in prefixed tokens in PrefixQuant, enabling static activation scales. These methods improve how activations fit the available numerical levels, while subsequent rounding still discards input-dependent information. \method{} targets the resulting output error after calibration, preserving the host's transformations and quantization settings.

\paragraph{Quantization error reconstruction.}
Reconstruction methods compensate quantization error with compact branches. \citet{zhang2024lqer} use activation-scaled weight-error decomposition in LQER, but its heuristic scaling does not directly solve expected output-error minimization. \citet{zhang2025qera} address this limitation through an analytical output-error formulation in QERA. Its weight-reconstruction target, however, does not explicitly include the output discrepancy introduced by activation rounding, so minimizing that target alone does not ensure activation-error recovery. Beyond the fitting objective, two-factor execution requires sequential products; when both factors run at low precision, intermediate quantization adds error and computation. To avoid this step, \citet{park2026serq} use a single compensation matrix on selected quantized activation channels. Yet its correction still receives only rounded inputs: activations sharing the same quantized vector receive the same adjustment, leaving their distinct output contributions unresolved. \method{} addresses this input limitation by fitting a residual-driven branch to the activation-rounding output difference under the frozen SERQ map.

\paragraph{Unrounded activations and residual correction.}
Accessing unrounded activations can recover distinctions unavailable from rounded inputs. \citet{scetbon2024lrc} use them in LRC while jointly optimizing quantized weights and a low-rank correction. This coupling means that adding its correction to an existing checkpoint requires revisiting the weight approximation. An analytical alternative is ARHQ: \citet{wang2026arhq} use activation-residual statistics to select a high-precision weight component, avoiding joint fitting of the weights and correction. However, the selected component is removed from the main-path weights before the remainder is quantized. Its construction therefore changes the quantization problem and does not preserve an already calibrated main path. Both methods access unrounded information through constructing the weight approximation, leaving a need for correction fitted independently after calibration. \method{} meets this requirement by using each input's rounding residual to drive an added low-rank branch. Its output-aware fit retains the host weights, quantizer, and existing reconstruction path, allowing the calibrated computation to be reused.

\endgroup

\section{Method}
\label{sec:method}
\begingroup
\setlength{\parskip}{3pt}

The preceding section identifies a limit of corrections that receive only rounded activations: they cannot distinguish inputs mapped to the same quantized value. ResCom addresses this limit after the original quantization pipeline has been calibrated. It passes the information removed by activation rounding to a small correction branch, then fits that branch to the missing layer output. The quantized path and any existing weight reconstruction remain frozen throughout calibration. Figure~\ref{fig:pipeline} shows how the branch is fitted and used.

\begin{figure}[!htbp]
\centering
\includegraphics[width=0.98\textwidth]{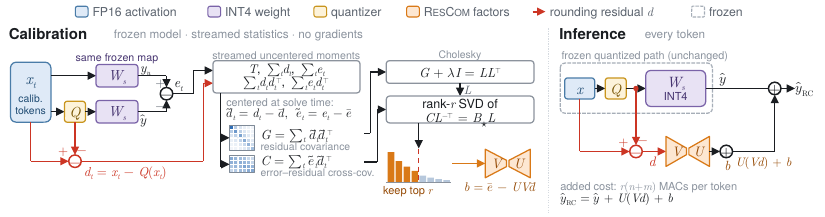}
\small
\setlength{\abovecaptionskip}{2pt}
\vspace{-6pt}
\caption{ResCom calibration and inference. Calibration pairs each activation rounding residual with the output contribution removed by rounding and fits a low-rank correction. At inference, the branch evaluates that correction for the current token and adds it to the frozen quantized output. The shared layer bias is omitted.}
\label{fig:pipeline}
\end{figure}

\subsection{Residual-conditioned correction}
\label{sec:setup}

\paragraph{Recovering the missing input distinction.}
For a layer with input width $n$, let the unrounded activation
$x\in\mathbb{R}^{n}$ denote the input to its calibrated quantizer,
expressed in the quantizer's coordinate system. The quantize--dequantize
operator $Q$ returns the dequantized activation $Q(x)\in\mathbb{R}^{n}$.
Different activations can share the same $Q(x)$ while having different
deviations from that value. A correction conditioned only on $Q(x)$
gives these inputs the same adjustment. ResCom therefore conditions
its branch on the activation rounding residual $d\in\mathbb{R}^{n}$,
defined as
\[
d=x-Q(x).
\]
The residual $d$ records the input information removed from the current
token. It is formed by elementwise subtraction at the quantizer, where
both $x$ and $Q(x)$ are available in the current forward pass.

\paragraph{Defining the correction target.}
The residual $d$ identifies the removed input component; the correction
branch is fitted to the resulting change in the layer output.
We obtain this target by evaluating $x$ and $Q(x)$ with the same frozen
linear map $W_s\in\mathbb{R}^{m\times n}$, where $m$ is the output width.
The map $W_s$ combines the dequantized main-path weights with any existing
linear reconstruction term and is evaluated through the host's existing
components. With the shared layer bias $c\in\mathbb{R}^{m}$, the
unrounded-input output $y_u$, the quantized-input output $\hat y$,
and the target output error $e$ are

\begin{equation}
y_u=W_sx+c,\qquad
\hat y=W_sQ(x)+c,\qquad
e=y_u-\hat y=W_sd.
\label{eq:layer}
\end{equation}
Because both outputs use the same frozen map and bias, $e$ isolates
the output contribution removed by activation rounding.

\paragraph{Adding a limited-rank branch.}
Computing $W_s d$ directly at inference would require a dense
multiplication. ResCom approximates this contribution with a rank-$r$
branch using the factors $V\in\mathbb{R}^{r\times n}$ and
$U\in\mathbb{R}^{m\times r}$. By concentrating this budget on residual
directions that most affect the layer output, the branch can recover
the dominant missing output contributions and thereby reduce
reconstruction error.
An intercept $b\in\mathbb{R}^{m}$ captures a constant output offset,
leaving the rank budget for token-dependent variation. The repaired
layer computes

\begin{equation}
\hat y_{\mathrm{RC}}=\hat y+U(Vd)+b.
\label{eq:repair}
\end{equation}
The factors are fixed after calibration, while $d$ changes with each token. Choosing those factors is therefore a question of which residual variations matter most to the output.

\subsection{Output-aware closed-form calibration}
\label{sec:fit}

The residual identifies what rounding removed, but its largest
variations need not produce the largest output errors: the frozen layer
may suppress some directions and amplify others. We therefore use an
output-reconstruction objective to select the residual directions
represented by the branch. Given $T$ calibration pairs $(d_t,e_t)$
from Eq.~\eqref{eq:layer}, we write the branch matrix as
$B=UV\in\mathbb{R}^{m\times n}$ and solve
\begin{equation}
\min_{\operatorname{rank}(B)\le r,\;b}
\sum_{t=1}^{T}\lVert e_t-Bd_t-b\rVert_2^2+\lambda\lVert B\rVert_F^2,
\qquad \lambda>0.
\label{eq:objective}
\end{equation}
The first term assigns rank to reductions in layer-output error. The ridge term stabilizes the fit when calibration data provide little evidence about a direction. The unpenalized intercept absorbs the mean error, leaving the low-rank factors to model differences between tokens.

\paragraph{Fitting from calibration statistics.}
ResCom collects residuals and their target output errors while running
the frozen model on calibration data. The mean residual
$\bar d\in\mathbb{R}^{n}$ and mean output error
$\bar e\in\mathbb{R}^{m}$ summarize their averages across calibration
tokens. Subtracting these means gives the centered residual
$\tilde d_t\in\mathbb{R}^{n}$ and centered output error
$\tilde e_t\in\mathbb{R}^{m}$:

\begin{equation}
\bar d=\frac{1}{T}\sum_{t=1}^{T}d_t,\qquad
\bar e=\frac{1}{T}\sum_{t=1}^{T}e_t,\qquad
\tilde d_t=d_t-\bar d,\qquad
\tilde e_t=e_t-\bar e.
\label{eq:rescom-centering}
\end{equation}

The centered pairs describe how each token differs from the calibration
averages. They determine the variations used to fit the low-rank map,
while the mean vectors are retained to recover the intercept after
fitting the factors. We summarize these centered pairs using the
unnormalized residual covariance matrix $G\in\mathbb{R}^{n\times n}$ and
the unnormalized output-error--residual cross-covariance matrix
$C\in\mathbb{R}^{m\times n}$:

\begin{equation}
G=\sum_{t=1}^{T}\tilde d_t\tilde d_t^{\top},\qquad
C=\sum_{t=1}^{T}\tilde e_t\tilde d_t^{\top}.
\label{eq:calibration_moments}
\end{equation}

The diagonal entries of $G$ measure the accumulated variation in each
residual coordinate, and its off-diagonal entries describe how residual
coordinates vary together. Each entry of $C$ measures how an output-error
coordinate varies with a residual coordinate. Thus, $G$ describes the
scale and correlations of the available residual variations, while $C$
relates those variations to their output effects. These statistics can
be obtained from accumulated moments without storing all calibration
tokens.

To select directions by their contribution to the output-error objective,
we first account for differences in residual scale and correlations.
We use the Cholesky decomposition $G+\lambda I_n=LL^\top$, where $I_n$
is the $n\times n$ identity matrix and $L$ is a lower-triangular factor.
This change of coordinates converts Eq.~\eqref{eq:objective}, after fitting
the intercept, into a low-rank approximation of $CL^{-\top}$
under the Frobenius norm.
We then compute its singular value decomposition (SVD),
$CL^{-\top}=\Phi\Sigma\Psi^\top$, where the columns of $\Phi$ and $\Psi$
are the left and right singular vectors, and $\Sigma$ contains the singular
values in decreasing order.
With the subscript $r$
denoting these components, the fitted branch is
\begin{equation}
U=\Phi_r\Sigma_r,\qquad
V=\Psi_r^\top L^{-1},\qquad
b=\bar e-UV\bar d.
\label{eq:closed_form}
\end{equation}
This keeps the directions that explain the greatest reduction in the
regularized output-error objective~\citep{izenman1975rrr}. Appendix~\ref{app:proof}
gives the derivation and numerical details.

All layer statistics are collected before any fitted branch is attached. Otherwise, an early correction would change the activations observed by later layers, so their calibration targets would no longer describe the same frozen model. Algorithm~\ref{alg:fit} summarizes this order.

\begin{algorithm}[H]
\caption{Closed-form calibration of ResCom}
\label{alg:fit}
\footnotesize
\begin{algorithmic}[1]
\Require Frozen calibrated model $f_\theta$ with fixed parameters $\theta$;
calibration set $\mathcal D$;
target layers $\mathcal L$; rank $r$; ridge coefficient $\lambda$
\Ensure Model $f_\theta$ augmented with one ResCom branch per target layer
\State Run $\mathcal D$ through frozen $f_\theta$; stream the moments of
$(d_t,e_t)$ for every $\ell\in\mathcal L$
\For{each target layer $\ell\in\mathcal L$}
  \State Compute $\bar d_\ell$, $\bar e_\ell$, $G_\ell$, and $C_\ell$ using
  Eqs.~\eqref{eq:rescom-centering} and~\eqref{eq:calibration_moments}
  \State Compute $L_\ell L_\ell^\top=G_\ell+\lambda I_n$, where $I_n$ is the
  $n\times n$ identity
  \State Compute the rank-$r$ SVD
  $C_\ell L_\ell^{-\top}\approx
  \Phi_{\ell,r}\Sigma_{\ell,r}\Psi_{\ell,r}^\top$
  \State Recover $U_\ell$, $V_\ell$, and $b_\ell$ using
  Eq.~\eqref{eq:closed_form}
\EndFor
\State Attach all fitted branches to $f_\theta$ using Eq.~\eqref{eq:repair}
\State \Return $f_\theta$
\end{algorithmic}
\end{algorithm}

\subsection{Residual alignment under a rank budget}
\label{sec:method_comparison}

The correction input determines which output errors a limited-rank branch
can represent. For the target $e_t=W_s d_t$ in Eq.~\eqref{eq:layer},
the residual directly parameterizes the missing output contribution.
The following result compares residual, unrounded, and rounded inputs
under the same rank budget and target.

\begin{proposition}[Residual alignment]
\label{prop:residual_alignment}
For the same calibration samples, let $q_t=Q(x_t)$,
$d_t=x_t-q_t$, and $e_t=W_s d_t$. For an input sequence
$z=(z_t)_{t=1}^{T}$, define its minimum unregularized correction error as
\begin{equation}
\mathcal R_r(z)
=\min_{\operatorname{rank}(B)\le r,\;b}
\sum_{t=1}^{T}\|e_t-Bz_t-b\|_2^2.
\label{eq:input_rank_error}
\end{equation}
Let $s_i$ be the singular values, in decreasing order, of the centered
target matrix $F=[\widetilde e_1,\ldots,\widetilde e_T]$.
For $0\le r\le\min(m,n,T)$,
\begin{equation}
\mathcal R_r(d)=\sum_{i>r}s_i^2
\le\min\{\mathcal R_r(x),\mathcal R_r(q)\}.
\label{eq:residual_rank_optimum}
\end{equation}
Let $P_r$ be the orthogonal projector onto the leading $r$ left singular
directions of $F$, with $P_0=0$. If $(\widehat B,\widehat b)$ solves
the regularized objective in Eq.~\eqref{eq:objective}, then
\begin{equation}
\sum_{t=1}^{T}\|e_t-\widehat B d_t-\widehat b\|_2^2
\le\mathcal R_r(d)+\lambda\|P_rW_s\|_F^2.
\label{eq:residual_ridge_bound}
\end{equation}
\end{proposition}

The complete proof is provided in Appendix~\ref{app:residual-alignment-proof}.

Proposition~\ref{prop:residual_alignment} shows that, for activation-rounding
error on the calibration samples, a residual-conditioned linear branch
attains the optimal unregularized reconstruction under the rank budget.
Its minimum error is no greater than that achievable with unrounded
or quantized activations under the same budget.
Equation~\eqref{eq:residual_ridge_bound} gives the corresponding
reconstruction-error bound for the ridge-regularized fit.
The matched-input controls in Table~\ref{tab:input} and the held-out
evaluations in Section~\ref{sec:exp} examine how this alignment translates
into model-level improvements.

\subsection{Integration and cost}
\label{sec:method_cost}

ResCom attaches at the input quantizer of a chosen linear projection. It forms the residual before the unrounded activation is discarded and adds the branch output after the existing linear computation. Scaling, rotation, or permutation introduced by the host pipeline is reflected in the coordinate system of $x$; ResCom does not need to undo those transformations. The branches can be attached to selected attention or feed-forward projections after calibration. The default configuration repairs attention projections, and Section~\ref{sec:ablation} evaluates other placements.

For an $n$-input, $m$-output layer, calibration stores $O(n^2+mn)$ statistics, independent of the number of calibration tokens; layers can be processed in groups to limit memory. This work is performed once, without gradients or updates to the quantized weights. At inference, forming the residual requires an $n$-element subtraction, and the branch uses $r(n+m)$ multiply-accumulates per token. Its factors and intercept store $r(n+m)+m$ scalars. The chosen rank therefore sets both the capacity of the output correction and its added inference cost.

\endgroup

\FloatBarrier
\begingroup
\setlength{\parskip}{4pt}
\setlength{\textfloatsep}{4pt plus 1pt minus 1pt}
\setlength{\intextsep}{4pt plus 1pt minus 1pt}
\setlength{\floatsep}{4pt plus 1pt minus 1pt}
\section{Experiments}
\label{sec:exp}
\subsection{Experimental setup}
\label{sec:setup-exp}
\paragraph{Models and comparisons.}
We evaluate Llama-3/3.2 1B--8B~\citep{grattafiori2024llama3}, LLaMA-2 7B/13B/70B~\citep{touvron2023llama2}, Mistral-7B-v0.3~\citep{jiang2023mistral}, and Gemma-3-1B/4B~\citep{gemma2025gemma3}. SERQ provides the main quantized checkpoints; every ResCom result is paired
with our evaluation before repair. PrefixQuant and FlatQuant test transfer to
other pipelines. W4A4 is the main setting; W4A8 appears in
Appendix~\ref{app:a8}. Within each pair, the checkpoint, tokenizer, evaluation
harness, and test examples are fixed, so the comparison isolates the fitted
branch rather than a change in quantization or evaluation protocol.

\paragraph{Calibration and metrics.}
The rank-32 attention branch uses 128 WikiText-2 windows of 2048 tokens and
$\lambda=0.01\operatorname{tr}(G)/n$. 
We evaluate language modeling with WikiText-2~\citep{merity2017pointer} and held-out C4~\citep{raffel2020t5} perplexity; general task performance with an eight-task zero-shot mean and 57-subject MMLU~\citep{hendrycks2021mmlu}; and reasoning and long-context generation with five-shot GSM8K~\citep{cobbe2021gsm8k} and three LongBench tasks~\citep{bai2024longbench}. Evaluation harness details~\citep{gao2024lmeval} and protocols appear in Appendix~\ref{app:accuracy}.

\subsection{Main comparison}
\label{sec:main-results}
\paragraph{Accuracy after residual repair.}
To test whether residual-conditioned correction improves an already
calibrated W4A4 model across model sizes and families, we compare each
repaired model with its paired quantized checkpoint.
Tables~\ref{tab:main}--\ref{tab:additional} show lower PPL across all evaluated
backbones. On Llama-3-8B, PPL falls from 7.708 to 7.147, MMLU gains 2.89
points, and zero-shot gains 1.91 points; Gemma-3-4B MMLU gains 1.62 points.
LLaMA-2-70B also improves in PPL and both accuracy aggregates. 
These paired improvements show that activation residuals remain useful for accuracy recovery after W4A4 calibration across model scales. Figure~\ref{fig:ppl} further shows that residual repair narrows the perplexity gap to full precision on Llama-3-8B and reduces perplexity on both corpora across the backbones shown.

\begin{table}[H]
\centering\scriptsize
\setlength{\tabcolsep}{1.8pt}
\renewcommand{\arraystretch}{0.84}
\caption{W4A4 WikiText-2 PPL and task accuracy. $\dagger$: published by
SERQ~\citep{park2026serq}; other rows are our evaluations. $\ddagger$: nominal
factor cost. Zero-shot averages eight tasks; bold compares paired rows.}
\label{tab:main}
\begin{tabular}{llccc ccc ccc ccc}
\toprule
& & \#Eff. & Train. & MAC
& \multicolumn{3}{c}{LLaMA-2 7B} & \multicolumn{3}{c}{LLaMA-2 70B}
& \multicolumn{3}{c}{LLaMA-3 8B} \\
\cmidrule(lr){6-8}\cmidrule(lr){9-11}\cmidrule(lr){12-14}
\#Bits & Method & bits/wt. & free & overhead
& PPL$\downarrow$ & 0-shot$\uparrow$ & MMLU$\uparrow$
& PPL$\downarrow$ & 0-shot$\uparrow$ & MMLU$\uparrow$
& PPL$\downarrow$ & 0-shot$\uparrow$ & MMLU$\uparrow$ \\
\midrule
FP16 & Full precision$^\dagger$ & 16 & \checkmark & --
& 5.47 & 64.09 & 41.83 & 3.53 & 70.54 & 65.44 & 6.13 & 67.16 & 62.13 \\
\midrule
\multirow{6}{*}{W4A4}
 & SmoothQuant$^\dagger$        & ${\sim}4$ & \checkmark & --     & 1.51e4 & 35.44 & 26.04 & -- & -- & -- & 4.71e5 & 36.34 & 25.17 \\
 & SmoothQuant, g128$^\dagger$  & 4.13 & \checkmark & --          & 7.49 & 57.15 & 30.40 & -- & -- & -- & 17.26 & 48.97 & 29.30 \\
 & QuaRot$^\dagger$             & ${\sim}4$ & \checkmark & 19.8\% & 6.15 & 59.53 & 33.58 & -- & -- & -- & 8.41 & 59.12 & 47.29 \\
 & SpinQuant$^\dagger$          & ${\sim}4$ & $\times$ & 19.8\%   & 6.00 & 61.00 & 34.80 & -- & -- & -- & 8.26 & 61.75 & 49.93 \\
\cmidrule(l){2-14}
 & SERQ             & 4.24 & \checkmark & 18.7\%      & 5.939 & \textbf{62.52} & 37.49 & 3.638 & 69.59 & 63.72 & 7.708 & 62.98 & 54.69 \\
 \rowcolor{gray!12}
 & \textbf{\method{}}       & 4.30$^\ddagger$ & \checkmark & 0.39\%$^\ddagger$   & \textbf{5.847} & 62.30 & \textbf{38.25} & \textbf{3.597} & \textbf{70.00} & \textbf{64.16} & \textbf{7.147} & \textbf{64.89} & \textbf{57.58} \\
\bottomrule
\end{tabular}

\vspace{3pt}
\begin{tabular}{llccc ccc ccc ccc}
\toprule
& & \#Eff. & Train. & MAC
& \multicolumn{3}{c}{LLaMA-2 13B} & \multicolumn{3}{c}{LLaMA-3.2 1B} & \multicolumn{3}{c}{LLaMA-3.2 3B} \\
\cmidrule(lr){6-8}\cmidrule(lr){9-11}\cmidrule(lr){12-14}
\#Bits & Method & bits/wt. & free & overhead
& PPL$\downarrow$ & 0-shot$\uparrow$ & MMLU$\uparrow$
& PPL$\downarrow$ & 0-shot$\uparrow$ & MMLU$\uparrow$
& PPL$\downarrow$ & 0-shot$\uparrow$ & MMLU$\uparrow$ \\
\midrule
FP16 & Full precision$^\dagger$ & 16 & \checkmark & --
& 4.88 & 66.53 & 52.04 & 9.75 & 54.82 & 36.76 & 7.81 & 62.66 & 54.06 \\
\midrule
\multirow{6}{*}{W4A4}
  &  SmoothQuant$^\dagger$         &  ${\sim}4$  &  \checkmark  &  --      &  1.32e4 & 34.53 & 23.85 &  1.53e5  &  35.59  &  24.37  &  3.73e4  &  35.72  &  23.41  \\
  &  SmoothQuant, g128$^\dagger$   &  4.13  &  \checkmark  &  --           &  6.31 & 61.28 & 39.83 &  69.22  &  40.04  &  24.43  &  53.33  &  44.17  &  27.31  \\
  &  QuaRot$^\dagger$              &  ${\sim}4$  &  \checkmark  &  19.8\%  &  5.41 & 62.55 & 47.25 &  13.17  &  50.03  &  26.64  &  9.73  &  55.76  &  44.75  \\
  &  SpinQuant$^\dagger$           &  ${\sim}4$  &  $\times$  &  19.8\%    &  5.20 & 64.80 & 47.80 &  13.47  &  48.95  &  26.38  &  10.15  &  56.88  &  42.42  \\
\cmidrule(l){2-14}
  &  SERQ              &  4.24  &  \checkmark  &  18.7\%       &  5.227 & 66.08 & 49.20 &  12.461  &  50.93  &  31.05  &  9.177  &  58.98  &  47.96  \\
 \rowcolor{gray!12}
  &  \textbf{\method{}}        &  4.30$^\ddagger$  &  \checkmark  &  0.39\%$^\ddagger$    &  \textbf{5.148} & \textbf{66.22} & \textbf{49.35} &  \textbf{12.071}  &  \textbf{51.02}  &  \textbf{31.21}  &  \textbf{8.964}  &  \textbf{60.14}  &  \textbf{49.29}  \\
\bottomrule
\end{tabular}
\end{table}

\begin{table}[H]
\centering
\caption{Cross-family controls and transfer. \textbf{(a)} Unquantized precision
is FP16 for Mistral and BF16 for Gemma; W4A4 is direct rounding without an
added algorithm. Accuracy is in percent; zero-shot averages eight tasks and
MMLU covers 57 subjects. \textbf{(b)} C4 is unseen during WikiText-2 fitting;
$\rho_{32}$ is the mean layerwise fit reduction, as defined in Appendix~A,
Eq.~\eqref{eq:rho}. Bold compares each SERQ/ResCom pair.}
\label{tab:additional}
\label{tab:backbones}
\vspace{2pt}
\begin{minipage}[t]{0.50\textwidth}
\centering
\textbf{(a) Cross-family accuracy}\par\vspace{2pt}
\tiny
\setlength{\tabcolsep}{1.45pt}
\renewcommand{\arraystretch}{0.86}
\begin{tabular}{@{}llrrrr@{}}
\toprule
Model & Method & WT2 & C4 & 0-shot & MMLU\\
\midrule
\multirow{4}{*}{\shortstack[l]{Mistral-7B\\v0.3}}
 & FP16 & 5.522 & 8.480 & 68.77 & 59.66\\
 & W4A4 & 5.77e3 & 3.69e3 & 35.32 & 24.43\\
 & SERQ & 5.663 & 8.244 & 67.10 & 56.12\\
 & ResCom & \textbf{5.617} & \textbf{8.188} & \textbf{67.13} & \textbf{56.73}\\
\midrule
\multirow{4}{*}{\shortstack[l]{Gemma-3-1B\\PT}}
 & BF16 & 14.088 & 19.980 & 56.69 & 24.81\\
 & W4A4 & 1.68e4 & 1.36e4 & 35.50 & 23.99\\
 & SERQ & 17.332 & 23.895 & 53.02 & 24.60\\
 & ResCom & \textbf{17.268} & \textbf{21.850} & \textbf{53.17} & \textbf{24.80}\\
\midrule
\multirow{4}{*}{\shortstack[l]{Gemma-3-4B\\PT}}
 & BF16 & 10.695 & 16.755 & 66.63 & 57.83\\
 & W4A4 & 8.47e4 & 6.40e4 & 35.22 & 24.75\\
 & SERQ & 12.522 & 17.347 & 62.12 & 50.55\\
 & ResCom & \textbf{12.191} & \textbf{16.557} & \textbf{62.96} & \textbf{52.17}\\
\bottomrule
\end{tabular}
\end{minipage}\hfill
\begin{minipage}[t]{0.49\textwidth}
\centering
\textbf{(b) Paired corpus transfer}\par\vspace{2pt}
\small
\setlength{\tabcolsep}{1.25pt}
\renewcommand{\arraystretch}{0.94}
\begin{tabular}{@{}lrrrrr@{}}
\toprule
& \multicolumn{2}{c}{WikiText-2} & \multicolumn{2}{c}{C4} & \\
\cmidrule(lr){2-3}\cmidrule(lr){4-5}
Backbone & SERQ & ResCom & SERQ & ResCom & $\rho_{32}$\\
\midrule
Llama-3-8B & 7.708 & \textbf{7.147} & 11.192 & \textbf{10.349} & .294\\
LLaMA-2 7B & 5.939 & \textbf{5.847} & 7.432 & \textbf{7.315} & .271\\
LLaMA-2 13B & 5.227 & \textbf{5.148} & 6.693 & \textbf{6.634} & .270\\
Llama-3.2-3B & 9.177 & \textbf{8.964} & 12.688 & \textbf{12.317} & .282\\
Llama-3.2-1B & 12.461 & \textbf{12.071} & 17.227 & \textbf{16.712} & .355\\
Mistral-7B-v0.3 & 5.663 & \textbf{5.617} & 8.244 & \textbf{8.188} & .343\\
\bottomrule
\end{tabular}
\end{minipage}
\end{table}

ResCom further lowers WikiText-2 and C4 perplexity on all three additional backbones after SERQ calibration has recovered usable performance from direct W4A4 rounding. 
Mistral improves all four aggregates, Gemma-3-4B gains 0.84 zero-shot and 1.62 MMLU points, and Gemma-3-1B lowers C4 from 23.895 to 21.850.
Together with the Llama results in Table~\ref{tab:main}, these paired gains
demonstrate the cross-family applicability of residual-conditioned
repair under a common rank-32 attention configuration. Across the
evaluated Llama, Mistral, and Gemma families, the rounding residual
remains useful after host calibration, providing a common correction
signal for further recovery of quantized model quality.
\begin{figure}[H]
\centering
\includegraphics[width=0.85\textwidth]{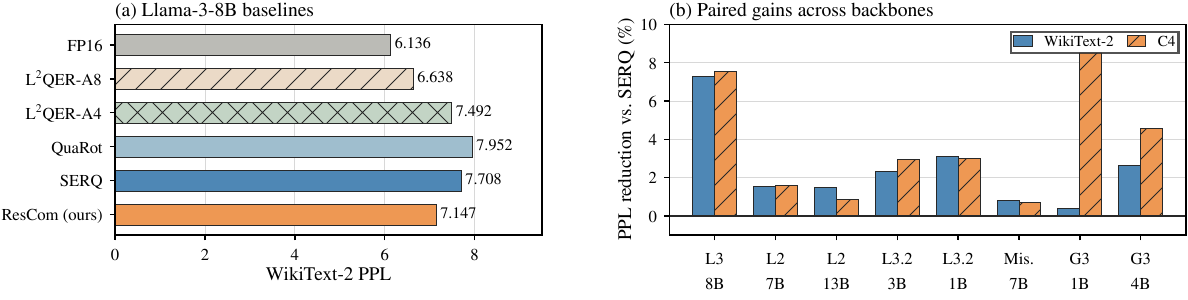}
\setlength{\abovecaptionskip}{2pt}
\caption{Perplexity comparisons. \textbf{(a)} Reproduced Llama-3-8B baselines;
L$^2$QER uses MXINT and QuaRot/SERQ use INT4 simulation. \textbf{(b)} Paired
W4A4 reductions across backbones. Table~\ref{tab:main} identifies published values.}
\label{fig:ppl}
\end{figure}

\paragraph{Transfer beyond the calibration setting.}
Table~\ref{tab:transfer} shows transfer in two directions. A WikiText-2 fit
lowers held-out C4 PPL on all five Llama backbones without refitting, including
11.192 to 10.349 on Llama-3-8B. The same construction improves SERQ,
PrefixQuant, and FlatQuant from 7.708/7.860/6.972 to 7.147/7.605/6.960. Thus the
branch transfers across corpora and attaches after distinct host pipelines;
targets are defined in Appendix~\ref{app:repro}.
Together, these results show that ResCom's accuracy
recovery extends beyond the calibration corpus and
can complement different PTQ pipelines while preserving
their calibrated weights and quantizers.

\begin{table}[H]
\centering
\caption{Transfer tests. \textbf{(a)} C4 PPL after WikiText-2 calibration,
without refitting. \textbf{(b)} WikiText-2 PPL on Llama-3-8B under three host
pipelines; only the correction changes within each pair.}
\label{tab:transfer}
\vspace{2pt}
\begin{minipage}[t]{0.66\textwidth}
\centering
\textbf{(a) Cross-corpus transfer}\par\vspace{2pt}
\scriptsize
\setlength{\tabcolsep}{2.4pt}
\begin{tabular}{@{}lrr@{\qquad}lrr@{}}
\toprule
Backbone & Host & ResCom & Backbone & Host & ResCom\\
\midrule
Llama-3-8B & 11.192 & \textbf{10.349} & Llama-3.2-3B & 12.688 & \textbf{12.317}\\
LLaMA-2 7B & 7.432 & \textbf{7.315} & Llama-3.2-1B & 17.227 & \textbf{16.712}\\
LLaMA-2 13B & 6.693 & \textbf{6.634} & & & \\
\bottomrule
\end{tabular}
\end{minipage}\hfill
\begin{minipage}[t]{0.32\textwidth}
\centering
\textbf{(b) Cross-pipeline transfer}\par\vspace{2pt}
\scriptsize
\setlength{\tabcolsep}{2.8pt}
\begin{tabular}{@{}lrr@{}}
\toprule
Pipeline / sites & Host & ResCom\\
\midrule
SERQ / attention & 7.708 & \textbf{7.147}\\
PrefixQuant / output & 7.860 & \textbf{7.605}\\
FlatQuant / output+down & 6.972 & \textbf{6.960}\\
\bottomrule
\end{tabular}
\end{minipage}
\end{table}

\paragraph{Reasoning and long-context generation.}
Strict GSM8K rises from 23.96\% to 36.16\% on Llama-3-8B and from 13.42\%
to 16.91\% on Llama-3.2-3B; both improve on all three LongBench tasks
(Table~\ref{tab:generation}). Across the two Gemma models, seven of the ten
reported metrics improve: Gemma-3-1B improves all three LongBench tasks, while
Gemma-3-4B improves GSM8K, QMSum, and RepoBench-P. 
These results show that activation residual recovery can translate into
gains in reasoning and long-context generation, extending its benefit
beyond perplexity.
\begin{table}[H]
\centering\scriptsize
\setlength{\tabcolsep}{3.8pt}
\renewcommand{\arraystretch}{0.92}
\caption{W4A4 reasoning and long-context generation. $\dagger$: published
SERQ values~\citep{park2026serq}; other SERQ rows are our paired evaluations.
GSM8K uses five shots; LongBench uses full sets with a 7500-token budget.}
\label{tab:generation}
\begin{tabular}{lllccccc}
\toprule
& & & \multicolumn{2}{c}{GSM8K 5-shot $\uparrow$} & \multicolumn{3}{c}{LongBench $\uparrow$}\\
\cmidrule(lr){4-5}\cmidrule(lr){6-8}
Model & Bits & Method & Flexible & Strict & QMSum & SAMSum & RepoBench-P\\
\midrule
\multirow{4}{*}{Llama-3-8B}
& FP16 & Full precision$^\dagger$ & 48.07 & 47.69 & 23.43 & 46.26 & 66.56\\
\cmidrule(l){2-8}
& W4A4 & L$^2$QER$^\dagger$ & 7.96 & 7.66 & 0.09 & 1.98 & 5.65\\
\cmidrule(l){2-8}
& W4A4 & SERQ & 24.56 & 23.96 & 21.12 & 42.31 & 61.03 \\
\rowcolor{gray!12}
&  & ResCom & \textbf{36.54} & \textbf{36.16} & \textbf{22.96} & \textbf{44.40} & \textbf{65.04} \\
\midrule
\multirow{4}{*}{Llama-3.2-3B}
& FP16 & Full precision$^\dagger$ & 26.08 & 25.93 & 23.94 & 42.98 & 64.42\\
\cmidrule(l){2-8}
& W4A4 & L$^2$QER$^\dagger$ & 3.56 & 2.73 & 14.92 & 33.87 & 41.99\\
\cmidrule(l){2-8}
& W4A4 & SERQ & 13.80 & 13.42 & 20.67 & 41.03 & 62.39 \\
\rowcolor{gray!12}
&  & ResCom & \textbf{17.21} & \textbf{16.91} & \textbf{21.72} & \textbf{41.17} & \textbf{62.64} \\
\midrule
Gemma-3-1B-PT & W4A4 & SERQ & \textbf{2.12} & \textbf{1.36} & 14.15 & 6.30 & 42.78 \\
\rowcolor{gray!12}
 & & ResCom & 1.82 & 0.99 & \textbf{15.10} & \textbf{7.33} & \textbf{43.92} \\
\midrule
Gemma-3-4B-PT & W4A4 & SERQ & 16.22 & 16.00 & 16.25 & \textbf{16.73} & 53.36 \\
\rowcolor{gray!12}
 & & ResCom & \textbf{23.05} & \textbf{22.82} & \textbf{20.75} & 8.17 & \textbf{59.36} \\
\bottomrule
\end{tabular}
\end{table}

\subsection{Ablation and sensitivity analysis}
\label{sec:ablation}
\begin{table}[H]
\centering
\renewcommand{\arraystretch}{0.92}
\begin{minipage}[t]{0.63\textwidth}
\centering\scriptsize
\caption{Matched-rank input ablation. Each feature uses the same target and
relative ridge rule. Columns give PPL and mean normalized regularized fit
reduction $\rho_r$ (Appendix~\ref{app:proof}).}
\label{tab:input}
\setlength{\tabcolsep}{3pt}
\begin{tabular}{lccccc}
\toprule
& rounded act. & \multicolumn{2}{c}{unrounded activation}
& \multicolumn{2}{c}{rounding residual} \\
\cmidrule(lr){3-4}\cmidrule(lr){5-6}
rank & PPL & PPL & $\rho_r$ & PPL & $\rho_r$ \\
\midrule
\phantom{00}8 & 8.210 & 7.241 & 0.015 & \textbf{7.203} & 0.147 \\
\phantom{0}16 & 8.299 & 7.237 & 0.018 & \textbf{7.167} & 0.211 \\
\phantom{0}32 & 8.327 & 7.242 & 0.021 & \textbf{7.147} & 0.294 \\
\phantom{0}64 & 8.362 & 7.235 & 0.025 & \textbf{7.128} & 0.387 \\
128           & 8.384 & 7.227 & 0.028 & \textbf{7.110} & 0.493 \\
256           & 8.417 & 7.229 & 0.032 & \textbf{7.085} & 0.623 \\
\midrule
\multicolumn{6}{l}{\emph{Mistral-7B-v0.3, one rank; unrepaired 5.663}} \\
\phantom{0}32 & 5.787 & 5.652 & 0.028 & \textbf{5.617} & 0.343 \\
\bottomrule
\end{tabular}
\end{minipage}\hfill
\begin{minipage}[t]{0.32\textwidth}
\centering\scriptsize
\caption{Attention rank sensitivity on Llama-3.2 at W4A4. Rank zero is unmodified.}
\label{tab:rank}
\setlength{\tabcolsep}{5pt}
\begin{tabular}{lrr}
\toprule
Rank & 1B & 3B \\
\midrule
0 & 12.46 & 9.18 \\
8 & 12.22 & 9.03 \\
16 & 12.16 & 9.02 \\
32 & 12.07 & 8.96 \\
64 & 11.99 & 8.92 \\
128 & 11.91 & 8.90 \\
256 & 11.81 & 8.87 \\
\bottomrule
\end{tabular}
\end{minipage}
\end{table}

\paragraph{Which input should the correction use?}
To test whether the gains come from residual conditioning rather than
additional branch capacity, we compare rounded activations, unrounded
activations, and rounding residuals under matched ranks, parameter counts,
and factor-MAC budgets.
At fixed rank, sites, target, and fit, the residual input beats the unrounded
input at every rank and the rounded input by a wide margin
(Table~\ref{tab:input}); Mistral follows the same order. The rounded input loses
within-bin distinctions, whereas the unrounded activation mixes the retained
$Q(x)$ with the missing $d$. Feeding $d$ isolates information absent from the
main path, explaining the matched-capacity advantage.

\begin{figure}[H]
\centering
\includegraphics[width=0.9\textwidth,trim=0bp 1.5bp 0bp 0bp,clip]{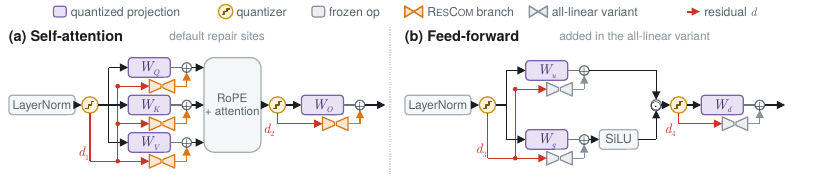}
\setlength{\abovecaptionskip}{2pt}
\caption{Repair sites. Orange marks default attention branches; gray marks added
feed-forward branches. Residuals $d_1$--$d_4$ follow the quantizers; $Q,K,V,O$
and $u,g,d$ label the projections.}
\label{fig:sites}
\end{figure}
\vspace{-8pt}
\begin{table}[H]
\centering
\scriptsize
\renewcommand{\arraystretch}{0.88}
\setlength{\abovecaptionskip}{2pt}
\setlength{\belowcaptionskip}{2pt}
\caption{Calibration stability and repair scope. \textbf{(a)} Mean
$\Delta\mathrm{PPL}=\mathrm{PPL}_{\mathrm{host}}-\mathrm{PPL}_{\mathrm{repaired}}$
with sample s.d. across fixed-evaluation calibration draws.
\textbf{(b)} Llama-3-8B choices; the channel control matches the default branch's
parameter and factor-MAC budgets.}
\label{tab:sensitivity}
\begin{minipage}[t]{0.47\textwidth}
\vspace{0pt}\centering
\textbf{(a) Calibration stability}\par\vspace{1pt}
\setlength{\tabcolsep}{3pt}
\begin{tabular}{lcr}
\toprule
Backbone & Draws & $\Delta$ PPL \\
\midrule
Llama-3-8B & 4 & $0.5604\pm0.0059$ \\
LLaMA-2 7B & 3 & $0.0916\pm0.0030$ \\
Mistral-7B-v0.3 & 4 & $0.0431\pm0.0024$ \\
\bottomrule
\end{tabular}
\end{minipage}\hfill
\begin{minipage}[t]{0.49\textwidth}
\vspace{0pt}\centering
\textbf{(b) Repair configuration}\par\vspace{1pt}
\setlength{\tabcolsep}{4pt}
\begin{tabular}{@{}lr@{\hskip 14pt}lr@{}}
    \toprule
    Repaired sites & PPL $\downarrow$ & Repaired sites / basis & PPL $\downarrow$ \\
    \midrule
    Bias only & 7.715 & Attention (default) & 7.147 \\
    Output & 7.672 & All linear & \textbf{7.070} \\
    Down & 7.624 & Budget-matched channels & \textbf{7.135} \\
    Gate + up & 7.688 & & \\
    \bottomrule
    \end{tabular}
\end{minipage}
\end{table}

\paragraph{Where and how much to repair.}
To identify effective repair locations, we compare corrections attached
to different attention and feed-forward projection groups under the
same calibration protocol.
Attention-only repair provides most of the gain; all-linear repair reaches
7.070 PPL with a larger budget (Figure~\ref{fig:sites},
Table~\ref{tab:sensitivity}(b), and Appendix Table~\ref{tab:sites}). 
To examine the effect of correction capacity, we vary the rank on Llama-3.2-1B and Llama-3.2-3B. Higher ranks generally reduce perplexity, with smaller incremental gains at larger ranks (Table~\ref{tab:rank}).
The budget-matched selected-channel branch reaches 7.135 PPL,
close to 7.147 for dense rank 32.
Together with the matched-input ablation, this result supports
residual conditioning as the key source of improvement across
the two correction parameterizations
(Appendix~\ref{app:budget}).

\paragraph{Sensitivity to calibration data.}
To assess sensitivity to the calibration data, we vary the calibration
draw and number of calibration windows while keeping the evaluation
protocol fixed.
Across repeated calibration draws, PPL reductions keep the same sign and sample
standard deviations stay below 0.006 (Table~\ref{tab:sensitivity}(a); Appendix
Table~\ref{tab:seeds}). On Llama-3-8B, 32/128/512 calibration windows yield
7.139/7.147/7.137 PPL (Appendix Table~\ref{tab:calib}).
Appendices~\ref{app:proof} and~\ref{app:a8} relate $\rho_r$ to absolute error
scale and downstream accuracy.
Together, these results indicate that the gains are stable across the
tested calibration draws and sample counts, supporting effective repair
with a small calibration set.

\subsection{Efficiency and implementation analysis}
\label{sec:cost-results}
We separately measure FP32 simulation, compact storage, and native MXFP4;
Appendix~\ref{app:deploy} gives kernel details.
\paragraph{One-time calibration cost.}
To assess the cost of adding ResCom to an existing quantized model,
we measure moment collection, fitting, and diagnostics on one B200.
Rank-32 attention repair takes 58--1353 seconds across the four
1B--13B models in Table~\ref{tab:calibtime}.
This one-time fitting procedure recovers accuracy without
gradient optimization or repeating the host quantization process.
\vspace{-3pt}
\paragraph{Prefill and decoding latency on B200.}
In FP32 simulation (Figure~\ref{fig:efficiency_latency}), ResCom adds
2.25\% to 4K prefill (1342.32 to 1372.46\,ms) while PPL falls from 7.708 to
7.147; cached decoding rises from 76.79 to 86.11\,ms/token.
The measured prefill increment shows that residual repair can recover
accuracy with limited additional cost during prompt processing.

\vspace{-3pt}
\paragraph{Peak memory and accuracy with compact weight storage.}
A separate compact implementation (Figure~\ref{fig:efficiency_memory})
uses 14.233\,GiB versus 19.686\,GiB for FP16 at 4096+32 tokens. The branch adds
0.103\,GiB (0.73\%) over SERQ while improving MMLU and zero-shot accuracy.
Residual repair therefore retains the memory savings of compact weight
storage while improving downstream accuracy.

\vspace{-3pt}
\paragraph{Native low-bit execution.}
Native MXFP4 repair lowers full WikiText-2 PPL by 0.132 across 224 projections,
adding 0.103\,GiB and 3.7\% latency over plain MXFP4 (Appendix Table~\ref{tab:native_mxfp4}).
Both variants use less memory and take longer than BF16.
These results demonstrate that residual correction can be integrated into
native low-bit computation while improving language-modeling performance.

\begin{figure}[H]
\centering
\vspace{-3pt}
\begin{minipage}[t]{0.4\textwidth}
  \centering
  \includegraphics[width=\linewidth,trim=0bp 5bp 0bp 0bp,clip]{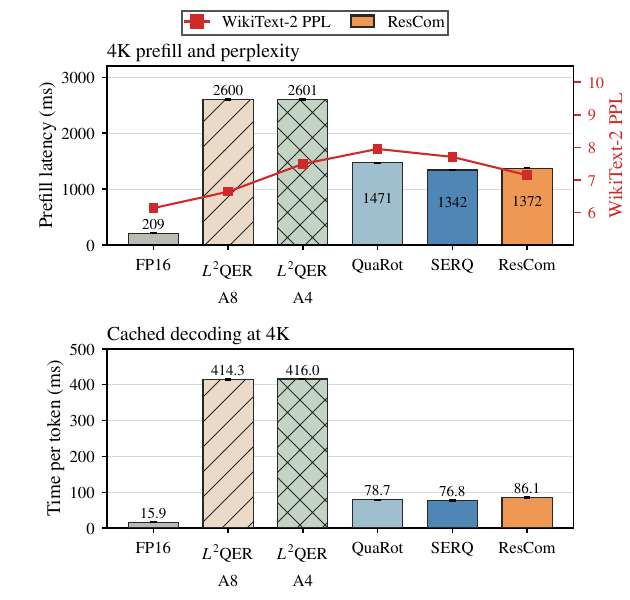}\par
  \vspace{0pt}
  {\footnotesize (a) Latency and perplexity\par}
\end{minipage}\hfill
\begin{minipage}[t]{0.4\textwidth}
  \centering
  \includegraphics[width=\linewidth,trim=0bp 5bp 0bp 0bp,clip]{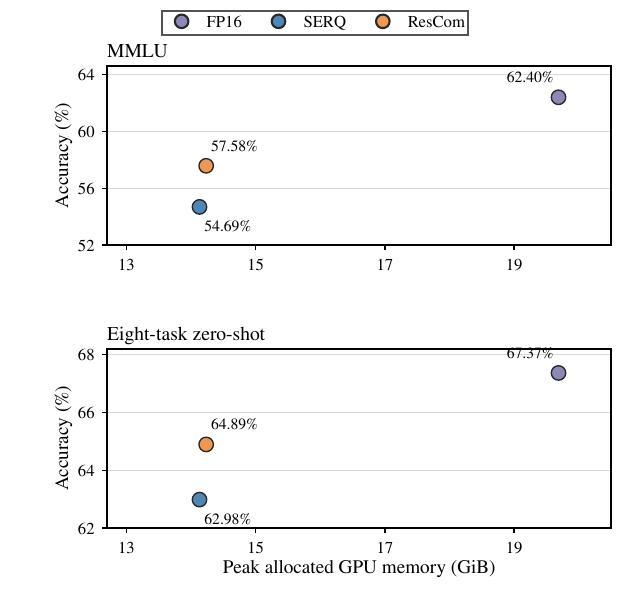}\par
  \vspace{0pt}
  {\footnotesize (b) Memory and accuracy\par}
\end{minipage}
\setlength{\abovecaptionskip}{2pt}
\caption{Efficiency and accuracy of Llama-3-8B on B200.
(a) 4K prefill latency and WikiText-2 perplexity (top), and cached
 decoding time per token (bottom), at batch size one. Quantized
 configurations use FP32 simulation; the FP16 reference uses FP16
 computation.
(b) MMLU (top) and eight-task zero-shot accuracy (bottom) versus peak
 allocated memory in a separate compact-storage experiment with 4096
 input and 32 generated tokens. Compact quantized weights are
 dequantized layer by layer for FP32 GEMM.}
\label{fig:efficiency_summary}
\makeatletter
\begingroup
\edef\@currentlabel{\thefigure(a)}
\label{fig:efficiency_latency}
\edef\@currentlabel{\thefigure(b)}
\label{fig:efficiency_memory}
\endgroup
\makeatother
\end{figure}

\vspace{-10pt}
\section{Conclusion}
\label{sec:conclusion}
\vspace{-4pt}
Low-bit activation quantization creates token-specific output errors that are
no longer identifiable from rounded activations alone. We introduced ResCom, a
low-rank branch conditioned on the rounding residual and fitted to layer-output
error by closed-form reduced-rank regression, while leaving the calibrated
weights and main quantized path unchanged. Across the evaluated 1B to 70B
backbones, the correction consistently lowers paired perplexity; on Llama-3-8B
it reduces WikiText-2 PPL from 7.708 to 7.147, raises MMLU by 2.89 points, and
improves five-shot GSM8K from 23.96\% to 36.16\%. Matched-input, cross-corpus,
and cross-pipeline studies connect these gains to the residual conditioning
signal, while deployment measurements quantify the additional latency and
memory. The evidence supports activation residuals as a useful complementary
signal for repairing an already calibrated low-bit model, with the attainable
gain depending on the backbone, task, and quantization path.
\endgroup

\clearpage
\section*{AI Use Statement}
Generative AI tools assisted with research discussion, experiment planning,
code implementation, figure preparation, prose editing, and checks of reported
results. The authors reviewed the assisted work against source code, raw
outputs, and cited literature and take responsibility for the paper's content.
\bibliographystyle{iclr2027_conference}
\bibliography{refs}
\appendix
\section*{Appendix}
\FloatBarrier
\section{Closed-form solution and fit diagnostics}
\label{app:proof}
Let $J(B,b)$ denote the objective in Eq.~\eqref{eq:objective}, with $B=UV$.
Define
\[
\bar d=\frac{1}{T}\sum_{t=1}^{T}d_t,
\qquad
\bar e=\frac{1}{T}\sum_{t=1}^{T}e_t,
\]
and let $\widetilde d_t=d_t-\bar d$ and
$\widetilde e_t=e_t-\bar e$. Minimizing over the unpenalized bias first gives
\begin{equation}
b=\bar e-B\bar d.
\label{eq:optimal_intercept}
\end{equation}
The input-dependent fit is described by the centered residual covariance and
residual--error cross-covariance
\begin{equation}
G=\sum_{t=1}^{T}\widetilde d_t\widetilde d_t^\top,
\qquad
C=\sum_{t=1}^{T}\widetilde e_t\widetilde d_t^\top.
\label{eq:moments}
\end{equation}
These statistics can be accumulated in a streaming pass, so individual
calibration vectors need not be retained.

Define the centered target energy
$E=\sum_{t=1}^{T}\|\widetilde e_t\|_2^2$ and
$G_\lambda=G+\lambda I_n$. After eliminating the intercept, the remaining
objective is
\begin{equation}
J(B,\bar e-B\bar d)
=E-2\operatorname{tr}(BC^\top)+\operatorname{tr}(BG_\lambda B^\top).
\label{eq:centered_objective}
\end{equation}
Here $\operatorname{tr}$ is the matrix trace. Since $\lambda>0$,
$G_\lambda$ is positive definite, and the unrestricted ridge solution is
\begin{equation}
B_\star=CG_\lambda^{-1}.
\label{eq:unrestricted_ridge}
\end{equation}
Factor $G_\lambda=LL^\top$ with a lower-triangular Cholesky factor $L$.
Completing the square gives
\begin{equation}
J(B,\bar e-B\bar d)
=E-\|B_\star L\|_F^2+\|(B-B_\star)L\|_F^2.
\label{eq:completed_square}
\end{equation}
The invertible change of variables $Z=BL$ preserves rank. Moreover,
\begin{equation}
B_\star L=CL^{-\top}.
\label{eq:whitened_map}
\end{equation}
Let the compact singular value decomposition be
\begin{equation}
CL^{-\top}=\Phi\Sigma\Psi^\top,
\label{eq:whitened_svd}
\end{equation}
with singular values in decreasing order. Retaining the leading $r$
components gives the reduced-rank solution
\begin{equation}
U=\Phi_{:,1:r}\Sigma_{1:r,1:r},\qquad
V=\Psi_{:,1:r}^{\top}L^{-1},\qquad
b=\bar e-UV\bar d.
\label{eq:truncate}
\end{equation}
Thus the rank constraint keeps the directions that explain the greatest
reduction in the regularized output-error objective, rather than those with
the largest residual magnitude. This is the reduced-rank regression
construction~\citep{izenman1975rrr} under a ridge-adjusted metric.

In implementation, neither $G_\lambda^{-1}$ nor $L^{-1}$ is formed
explicitly. Cholesky and triangular solves are used to obtain
$CL^{-\top}$ and $V$, improving numerical stability.
The relative ridge rule $\lambda=\gamma\operatorname{tr}(G)/n$,
with $\gamma>0$, gives a positive coefficient whenever
$\operatorname{tr}(G)>0$. If $G=0$, all centered residuals vanish
and hence $C=0$. In this degenerate case, we define the fitted
correction by $B=0$ and $b=\bar e$, without a Cholesky factorization.
This is a minimizer of Eq.~\eqref{eq:objective} for every positive
ridge coefficient, since the centered data provide no
input-dependent variation to fit.

Let $\sigma_i$, $i=1,\ldots,s$, be the ordered singular values on the
diagonal of $\Sigma$, and let $J_r$ be the minimum objective at rank $r$.
For $E>0$, define the normalized regularized fit reduction $\rho_r$ by
\begin{equation}
 \rho_r=\frac{\sum_{i=1}^r\sigma_i^2}{E},\qquad
 J_r=E(1-\rho_r),\qquad 0\le\rho_r\le1.
\label{eq:rho}
\end{equation}
The numerator is the regularized objective reduction from the retained
singular components. The bounds follow from nonnegativity of the objective
and the squared singular values.
The penalty acts on $B$, not independently on $U,V$. The unpenalized calibration
error reduction of this fitted map is $E\rho_r+\lambda\|B\|_F^2$;
$\rho_r$ therefore describes the regularized objective. Reported model-level
values are arithmetic means over repaired layers, not energy-weighted global
bounds. If $E=0$, the ratio is undefined; the constant bias $b=\bar e$ suffices
on the calibration samples. Tied singular values may make the rank-constrained minimizer nonunique.

\subsection{Error scale and held-out correction}
\label{app:scale}
Let $\bar y_{\rm ref}=T^{-1}\sum_{t=1}^T y_{{\rm ref},t}$ and
$H=\sum_{t=1}^T\|y_{{\rm ref},t}-\bar y_{\rm ref}\|_2^2>0$ be the
centered reference-output energy. The scale-aware quantity
$\alpha_r=E\rho_r/H$ combines target energy and normalized fit reduction. Under the idealized rescaling
$(d_t,e_t)\mapsto(\epsilon d_t,\epsilon e_t)$ and the implemented relative
ridge $\lambda=\gamma\operatorname{tr}(G)/n$, with $\gamma>0$,
$\epsilon\ne0$, and $\operatorname{tr}(G)>0$, the matrices
$G,C$ and $\lambda$ scale by $\epsilon^2$.
Hence $B_\star$ is unchanged,
$B_\star L$ scales by $|\epsilon|$, and $E\rho_r$ scales by $\epsilon^2$.
The ratio $\rho_r$ is invariant, while $\alpha_r$ scales by $\epsilon^2$
for a fixed reference energy $H$. Changing bit width is not generally a
uniform rescaling, but this identity explains why a normalized spectrum
cannot by itself identify large accuracy gains.

For a fitted correction $p_t=UVd_t+b$ on held-out tokens, define
\begin{equation}
\eta_{\rm val}
=\frac{\sum_t\|e_t\|_2^2-\sum_t\|e_t-p_t\|_2^2}
       {\sum_t\|y_{{\rm ref},t}\|_2^2}.
\label{eq:heldout-energy}
\end{equation}
The denominator must be positive. Positive $\eta_{\rm val}$ means lower
held-out layer error, with calibration factors and bias held fixed.
A negative value detects local overcorrection. It is not a bound on language
model loss: errors propagate through later nonlinear layers. This diagnostic
uses validation inputs, separately from the fixed test windows used for PPL.

\subsection{What the capacity controls identify}
\label{app:budget}
The rounded-input, unrounded-input and residual-input branches in
Table~\ref{tab:input} all have $r(n+m)+m$ scalar parameters and $r(n+m)$
factor MACs for each $m\times n$ layer. The unrounded-input branch is fitted
to the same target, so the comparison isolates which input signal is supplied
to an otherwise matched low-rank map.

A channel-restoration control with $k$ selected input channels uses $mk+m$
parameters and $mk$ MACs, excluding index storage and gather operations.
Exact per-layer matching requires $k=r(n+m)/m$. The comparison below chooses
$k$ separately for square and rectangular projections so that its total
parameter and factor-MAC budget equals that of rank-32 ResCom.

The strengthened Llama-3-8B control uses the same 128 training windows and
seed zero as the original fit. For each layer, input-channel indices
$j\in\{1,\ldots,n\}$ are ranked by
$\sum_t d_{t,j}^2\|(W_s)_{:,j}\|_2^2$, including the stored SERQ compensator.
The control restores $(W_s)_{:,I}d_I$ on the selected index set
$I\subseteq\{1,\ldots,n\}$, $|I|=k$, and fits a mean
bias. Query/output projections use $k=64$ and key/value projections use
$k=160$. Both parameterizations use 27,590,656 scalar parameters and
27,262,976 factor MACs per token across the 128 repaired projections;
the channel implementation additionally stores integer indices and gathers
activations. The selected-channel and dense low-rank bases reach 7.1351 and
7.1465 PPL, respectively. Their similar performance shows that the benefit of
the residual signal is not tied to one basis: the dense form provides a
fixed-rank factorization, while the sparse form provides coordinate-level
selection with indexed access.

\subsection{Proof of Proposition~\ref{prop:residual_alignment}}
\label{app:residual-alignment-proof}

\begin{proof}
Fix the calibration samples and an integer rank budget
$0\le r\le\min(m,n,T)$.
We first eliminate the intercept, then derive a lower bound that
applies to each candidate input, and finally show that residual
input attains this bound.

\paragraph{Centering the inputs and targets.}
For any input sequence $z\in\{x,q,d\}$, define
\[
\bar z=\frac{1}{T}\sum_{t=1}^{T}z_t,\qquad
\bar e=\frac{1}{T}\sum_{t=1}^{T}e_t,\qquad
\tilde z_t=z_t-\bar z,\qquad
\tilde e_t=e_t-\bar e.
\]
Let
\[
Z=[\tilde z_1,\ldots,\tilde z_T],\qquad
F=[\tilde e_1,\ldots,\tilde e_T].
\]
For fixed $B$, the prediction error decomposes as
\[
e_t-Bz_t-b
=
(\tilde e_t-B\tilde z_t)+(\bar e-B\bar z-b).
\]
The first term sums to zero over the calibration samples.
Consequently, the cross term vanishes when the squared errors
are summed, giving
\[
\sum_{t=1}^{T}\|e_t-Bz_t-b\|_2^2
=
\|F-BZ\|_F^2+T\|\bar e-B\bar z-b\|_2^2.
\]
The minimizing intercept is therefore
\[
b=\bar e-B\bar z,
\]
and the minimum unregularized correction error becomes
\[
\mathcal R_r(z)=
\min_{\operatorname{rank}(B)\le r}\|F-BZ\|_F^2.
\]
The same centering identity applies to the regularized objective,
because its penalty depends on $B$ and leaves the intercept
unpenalized.

\paragraph{A common lower bound for every input.}
For every admissible $B$,
\[
\operatorname{rank}(BZ)
\le\operatorname{rank}(B)\le r.
\]
Thus, the centered prediction $BZ$ belongs to the set of matrices
of rank at most $r$. Minimizing over all matrices in this larger
set gives a lower bound on the error achievable from any
particular input sequence:
\[
\|F-BZ\|_F^2
\ge
\min_{\operatorname{rank}(M)\le r}\|F-M\|_F^2
=
\sum_{i>r}s_i^2.
\]
The equality is the best rank-$r$ approximation formula from the
Eckart--Young--Mirsky theorem, where $s_i$ are the singular values
of $F$ in nonincreasing order. It follows that
\[
\mathcal R_r(z)\ge\sum_{i>r}s_i^2,\qquad z\in\{x,q,d\}.
\]

\paragraph{Attaining the lower bound with residual input.}
The activation-rounding target satisfies $e_t=W_sd_t$ for every
calibration sample. Averaging this identity gives
\[
\bar e=W_s\bar d.
\]
Subtracting the means therefore yields
\[
\tilde e_t=W_s(d_t-\bar d)=W_s\tilde d_t.
\]
With
\[
D=[\tilde d_1,\ldots,\tilde d_T],
\]
we obtain the centered relation
\[
F=W_sD.
\]

Choose an ordered singular value decomposition of $F$, and let
$P_r$ be the orthogonal projector onto its leading $r$ left
singular directions, with $P_0=0$. The matrix $P_rF$ retains
the leading $r$ singular components of $F$ and satisfies
\[
\|F-P_rF\|_F^2=\sum_{i>r}s_i^2.
\]
Consider the residual-input correction
\[
B_0=P_rW_s,\qquad b_0=\bar e-B_0\bar d.
\]
This correction is admissible because
\[
\operatorname{rank}(B_0)
=\operatorname{rank}(P_rW_s)
\le\operatorname{rank}(P_r)\le r.
\]
Its centered prediction is exactly the truncated reconstruction:
\[
B_0D=P_rW_sD=P_rF.
\]
Using the centering identity,
\[
\sum_{t=1}^{T}\|e_t-B_0d_t-b_0\|_2^2
=
\|F-P_rF\|_F^2
=
\sum_{i>r}s_i^2.
\]
Residual input therefore attains the common lower bound:
\[
\mathcal R_r(d)=\sum_{i>r}s_i^2.
\]
Combining this equality with the bounds for $z=x$ and $z=q$
proves Eq.~\eqref{eq:residual_rank_optimum}:
\[
\mathcal R_r(d)=\sum_{i>r}s_i^2
\le\min\{\mathcal R_r(x),\mathcal R_r(q)\}.
\]

\paragraph{The bound for the regularized fit.}
Let $(\widehat B,\widehat b)$ minimize the regularized objective
in Eq.~\eqref{eq:objective}, with the same rank constraint and $\lambda>0$.
Because the intercept is unpenalized, the centering identity gives
\[
\widehat b=\bar e-\widehat B\bar d.
\]
The pair $(B_0,b_0)$ constructed above is feasible for this
regularized problem. Optimality of
$(\widehat B,\widehat b)$ therefore implies
\[
\begin{aligned}
&\sum_{t=1}^{T}\|e_t-\widehat B d_t-\widehat b\|_2^2
 +\lambda\|\widehat B\|_F^2
\\
&\qquad\le
\sum_{t=1}^{T}\|e_t-B_0d_t-b_0\|_2^2
 +\lambda\|B_0\|_F^2
\\
&\qquad=
\mathcal R_r(d)+\lambda\|P_rW_s\|_F^2.
\end{aligned}
\]
The penalty on the left is nonnegative. Removing it gives
\[
\sum_{t=1}^{T}\|e_t-\widehat B d_t-\widehat b\|_2^2
\le
\mathcal R_r(d)+\lambda\|P_rW_s\|_F^2,
\]
which is Eq.~\eqref{eq:residual_ridge_bound}.

The argument also covers the boundary cases.
For $r=0$, the construction gives $B_0=0$,
$b_0=\bar e$, and $\mathcal R_0(d)=\|F\|_F^2$.
When $r\ge\operatorname{rank}(F)$, the discarded singular
values are zero and $P_rF=F$.
If singular values are tied or zero, any ordered orthonormal
singular-vector basis defines a valid $P_r$ and gives the same
reconstruction error; the regularized bound holds for that
chosen projector.
In particular, if $F=0$, the correction $B=0$, $b=\bar e$
attains zero fitting error and zero ridge penalty.
\end{proof}

\FloatBarrier
\section{Post-calibration repair versus residual-aware weight splitting}
\label{app:priorart}
ARHQ~\citep{wang2026arhq} and ResCom both preserve information unavailable from
the rounded activation alone, but they solve different integration problems.
ARHQ starts from a full-precision weight
$W\in\mathbb R^{m\times n}$, chooses a rank-$r$ component
$W_{\rm split}=B_aA_a^\top$ with
$A_a\in\mathbb R^{n\times r}$ and $B_a\in\mathbb R^{m\times r}$, and then
quantizes the modified remainder. With weight quantizer $Q_w$ and the shared
bias omitted, its deployed layer is
\begin{equation}
\hat y_{\rm ARHQ}
=Q_w(W-W_{\rm split})Q(x)+W_{\rm split}x.
\label{eq:arhq_operator}
\end{equation}
The high-precision branch is therefore coupled to a newly quantized main path.

ResCom instead starts from an already calibrated effective map $W_s$ and keeps
that map fixed. Its deployed layer is
\begin{equation}
\hat y_{\rm RC}=W_sQ(x)+Bd+b,\qquad d=x-Q(x),
\label{eq:rescom_operator}
\end{equation}
where $B=UV$ is fitted to the activation-only target $e=W_sd$
in the default SERQ setting. The deployed residual-conditioned
branch keeps the calibrated main operator fixed.
The same reduced-rank regression procedure also supports other
fixed output-error targets. Our PrefixQuant and FlatQuant extensions
use references that include weight-side error, as detailed in
Appendix~\ref{app:repro}, without repeating their weight calibration.

An algebraic expansion makes the distinction explicit. Adding and subtracting
$Q_w(W)Q(x)$ in Eq.~\eqref{eq:arhq_operator} gives
\begin{equation}
\hat y_{\rm ARHQ}
=Q_w(W)Q(x)+W_{\rm split}d+\Delta_wQ(x),
\qquad
\Delta_w=Q_w(W-W_{\rm split})+W_{\rm split}-Q_w(W).
\label{eq:arhq_expansion}
\end{equation}
The term $W_{\rm split}d$ exposes a residual-dependent component, while
$\Delta_wQ(x)$ records the accompanying change to the quantized main path.
Because quantization is nonlinear, optimizing the split and requantizing the
remainder generally changes both terms together. ResCom has no corresponding
main-path change: it learns $B$ from the frozen host layer's observed
residual-to-output relation. Thus the two methods differ in optimization stage,
branch input, fitted target, and whether deployment replaces the calibrated
main operator; their shared access to unrounded information does not make their
deployed corrections equivalent.

\section{Reproduction and implementation details}
\label{app:repro}
For completeness, the signal limitation used in the introduction can be
stated directly. For two inputs $x_1,x_2\in\mathbb R^n$ with
$Q(x_1)=Q(x_2)$ but $W_s(x_1-x_2)\ne0$, any deterministic correction
conditioned only on $Q(x)$ must be identical for both inputs, although their
required corrections differ. Residual input distinguishes them. This is a
statement about the information available to the branch; an unrounded-input
branch also has that information and is tested as a matched-rank control.

The representation $\hat y=W_sQ(x)+c$ in Eq.~\eqref{eq:layer}
assumes that the frozen main path and any existing linear
compensator act on the same quantize--dequantize input $Q(x)$.
All weights and activations are expressed in the coordinate system
of the calibrated input quantizer, including any host scaling,
rotation, or permutation.
For the SERQ instantiation,
$W_s=\widehat W+RS$, with $W_c=RS$, where
$S\in\mathbb R^{k\times n}$ selects $k$ salient channels.
Here $\widehat W\in\mathbb R^{m\times n}$ and
$R\in\mathbb R^{m\times k}$ denote the fixed numerical weight
values used by the deployed paths, including dequantization
where applicable. The combined matrix $W_s$ is never materialized
for inference.

This single-map representation does not automatically apply
to hosts with separate activation quantizers for different paths
or an additional output quantizer. In such cases, the target must
be defined from the actual frozen reference and host operators,
rather than assuming $e=W_sd$.
For a linear reference in the same coordinates, we write
$y_{\rm ref}=W_{\rm ref}x+c$ with
$W_{\rm ref}\in\mathbb R^{m\times n}$ and define
$e=y_{\rm ref}-\hat y$.

For the main SERQ experiment, the GPTQ checkpoint is loaded without another
weight-quantization call. Its unmodified WikiText-2 perplexity is
7.7077 and its C4 perplexity is 11.1922. The original
attention-only repair gives 7.1465 on WikiText-2. The fit reference
includes SERQ's stored selected-column compensator. All repaired layers use
phase-three asymmetric activation quantization in the evaluated configuration.

All calibration statistics are collected from the frozen,
unrepaired host model before any ResCom branch is attached.
This choice defines a common calibration reference across layers.
After the branches are attached, upstream corrections can change
the activations and rounding residuals encountered by downstream
layers. The closed-form solution is therefore optimal for the
specified layerwise objective on the collected calibration inputs;
it does not guarantee optimality under the input distribution
induced by the augmented model, or a reduction in end-to-end
language-model loss.

Uncentered moments and means are accumulated and then centered at solve time.
In the SERQ runs, products are computed in FP32, accumulated into FP64
statistics, and solved in FP64 before factors are cast to the loaded model
dtype. The artifact-bound timing run confirms FP32 loading for the existing
accuracy implementation. The 16-bit effective-width convention in
Table~\ref{tab:main} is therefore nominal storage accounting, not its in-memory
runtime footprint. For Llama-3-8B, the attention-only branch contains
27,590,656 parameters including biases: 52.625 MiB at 16 bits and 105.25 MiB
at 32 bits. Its factor MAC count is 0.3906\% of decoder-block linear MACs;
the all-linear variant is 1.2019\%. The ratios are model-specific.

The original attention artifact required 252.8 seconds for collection and
fitting on one B200, including 142.1 seconds of collection. Spectrum
computation in subsequent diagnostic runs adds work. The cost is a forward
calibration pass with dense moment accumulation and a factorization per layer,
not merely an ordinary inference pass.

The activation-only target in Eq.~\eqref{eq:layer} is the default
SERQ setting, with $W_{\rm ref}=W_s$ and hence $e=W_sd$.
PrefixQuant and FlatQuant instead use their pre-rounding weight
reference, including weight-side error. For a frozen host output
$\hat y=W_sQ(x)+c$, their target is
\[
e=W_{\rm ref}x-W_sQ(x)
 =W_sd+(W_{\rm ref}-W_s)x,
\qquad d=x-Q(x).
\]
These are extensions of the calibration target, while retaining
the same residual-conditioned branch and frozen host parameters.
The regression objective in Eq.~\eqref{eq:objective} and the
closed-form solution in Appendix~\ref{app:proof} remain valid for
these fixed targets; their derivation does not require $e=W_sd$.
By contrast, Proposition~1 relies on $e_t=W_sd_t$ and is claimed
only for the activation-only setting. We do not extend its
residual-alignment guarantee to targets containing weight-side
error.

PrefixQuant output-only rank-64 repair gives 7.6046.
FlatQuant's 6.9600 uses output and down projections; its all-linear
variant instead gives 6.9790. These distinct configurations are
kept separate in the comparisons. The PrefixQuant and FlatQuant
runs use FP32 moment accumulators.

\FloatBarrier
\section{Activation-width results}
\label{app:offdomain}
The paired WikiText-2/C4 results across Llama and Mistral backbones are reported
in the main-paper Table~\ref{tab:backbones}. The C4 documents are fixed across
comparisons and unseen during WikiText-2 calibration.
\subsection{Within-backbone width sweep}
\label{app:bits}
\begin{table}[!htbp]
\centering\small
\caption{Activation-width sweep on Llama-3-8B at fixed quantized weights. The sixteen-bit row is the same quantized host with activation quantization disabled, not the original full-precision model.}
\label{tab:bits}
\begin{tabular}{lccc}
\toprule
Activation width & Pipeline & \method{} & change \\
\midrule
4 bits  & 7.708 & \textbf{7.147} & $-7.28\%$ \\
6 bits  & 6.552 & 6.545 & $-0.10\%$ \\
8 bits  & 6.503 & 6.502 & $-0.01\%$ \\
16 bits & 6.499 & \tbd  & $0$ \\
\bottomrule
\end{tabular}
\end{table}

The A16 control retains quantized weights; it is not the original FP16 model.
For Llama-3-8B it gives 6.4991 perplexity. The rank-32 A4 repair recovers about
46.4\% of the A16-to-A4 perplexity gap. This is a ratio of measured perplexities,
not a layerwise error-energy fraction.
\subsection{Cross-backbone activation-width analysis}
\label{app:a8}
Table~\ref{tab:a8} compares the same checkpoints, repair rank, and attachment
sites at four- and eight-bit activation precision.
\begin{table}[!htbp]
\centering\small
\caption{Activation-width sensitivity with SERQ. We retain the weights,
repair rank, and attachment sites while changing the activation width.
Each relative change is measured against SERQ at the same activation precision.}
\label{tab:a8}
\begin{tabular}{lccccccc}
\toprule
& \multicolumn{4}{c}{Four-bit activations} & \multicolumn{3}{c}{Eight-bit activations} \\
\cmidrule(lr){2-5}\cmidrule(lr){6-8}
Backbone & SERQ & \method{} & change & $\rho_{32}$
         & SERQ & change & $\rho_{32}$ \\
\midrule
Llama-3-8B       & 7.708 & 7.147 & $-7.28$\% & 0.294 & 6.503 & $-0.01$\% & 0.289 \\
LLaMA-2 7B       & 5.939 & 5.847 & $-1.55$\% & 0.271 & 5.594 & $-0.01$\% & 0.265 \\
LLaMA-2 13B      & 5.227 & 5.148 & $-1.51$\% & 0.270 & 4.977 & $-0.01$\% & 0.261 \\
LLaMA-3.2 3B     & 9.177 & 8.964 & $-2.32$\% & 0.282 & 8.185 & $-0.02$\% & 0.282 \\
LLaMA-3.2 1B     & 12.461 & 12.071 & $-3.13$\% & 0.355 & 10.411 & $-0.02$\% & 0.356 \\
Mistral-7B-v0.3  & 5.663 & 5.617 & $-0.81$\% & 0.343 & 5.405 & $-0.01$\% & 0.341 \\
\bottomrule
\end{tabular}
\end{table}

All six backbones in Table~\ref{tab:a8} have less than 0.1\% relative
perplexity change at A8, whereas A4 produces the larger paired reductions.
The near-preserved $\rho_{32}$ values show that normalized fit quality alone
does not determine model-level gain: when activation rounding is already
small, the absolute correction target is also small. The complete
additional-family and generation measurements appear in
Tables~\ref{tab:additional_a8} and~\ref{tab:generation_a8}.

\FloatBarrier
\section{Attachment sites and calibration sensitivity}
\label{app:sites}
\begin{table}[!htbp]
\centering\small
\caption{Attachment-site and rank ablations on Llama-3-8B. Rank zero means a fitted bias only. Dashes denote unmeasured configurations. Site groups differ in parameter count, so they are not equal-budget comparisons.}
\label{tab:sites}
\begin{tabular}{llcccc}
\toprule
Pipeline & Repaired layers & $r=0$ & $r=16$ & $r=32$ & $r=64$ \\
\midrule
SERQ & all linear layers      & \tbd & \tbd & \textbf{7.070} & \tbd \\
SERQ & attention projections & 7.715 & 7.167 & 7.147 & 7.128 \\
SERQ & down projection       & \tbd & \tbd & 7.624 & \tbd \\
SERQ & output projection     & \tbd & \tbd & 7.672 & \tbd \\
SERQ & gate and up projections & \tbd & \tbd & 7.688 & \tbd \\
PrefixQuant & output projection  & 7.790 & \tbd & 7.646 & \textbf{7.605} \\
\bottomrule
\end{tabular}
\end{table}

The attention, down, and gate/up groups are fitted against the same unrepaired
model, then loaded together. The combined Llama-3-8B artifact contains 224
branches and yields 7.0697 WikiText-2 and 10.2251 C4
perplexity, 58.16 MMLU, and 64.93 zero-shot average. Its larger parameter and
MAC budget separates it from the default attention-only configuration.

\subsection{Calibration draws and sample count}
\label{app:seeds}
\begin{table}[!htbp]
\centering\small
\caption{Calibration-draw variability with test windows fixed. Positive $\Delta$ is host minus repaired perplexity. Standard deviations describe three or four draws as indicated by the final column; they are not test-set confidence intervals.}
\label{tab:seeds}
\begin{tabular}{lccccc}
\toprule
Backbone & Control & Repaired, mean & $\Delta$ mean & $\Delta$ s.d. & Helps on \\
\midrule
Llama-3-8B       & 7.708 & 7.147 & $+0.5604$ & 0.0059 & 4/4 \\
LLaMA-2 7B       & 5.939 & 5.847 & $+0.0916$ & 0.0030 & 3/3 \\
Mistral-7B-v0.3  & 5.663 & 5.620 & $+0.0431$ & 0.0024 & 4/4 \\
\bottomrule
\end{tabular}
\end{table}

Evaluation windows are fixed across calibration draws. The positive
perplexity effects on Llama-3-8B, LLaMA-2 7B, and Mistral preserve their sign
across the completed draws.
These standard deviations measure fit variability, not paired downstream-task
confidence intervals.
\begin{table}[h]
\centering\small
\caption{Calibration sample-count sensitivity on Llama-3-8B. Rank 32,
attention projections, seed zero; each control reproduces 7.7077.}
\label{tab:calib}
\begin{tabular}{rccc}
\toprule
2048-token windows & WikiText-2 & C4 & Mean $\rho_{32}$\\
\midrule
32 & 7.1387 & 10.3471 & 0.2934\\
128 & 7.1465 & 10.3493 & 0.2942\\
512 & 7.1374 & 10.3522 & 0.2938\\
\bottomrule
\end{tabular}
\end{table}

\paragraph{Large-model and additional-family completion.}
The LLaMA-2-70B comparison uses a separately verified GPTQ checkpoint with
128-channel weight groups and rank-32 corrections on its 320 attention
projections. An independent single-writer reload reproduces WikiText-2
PPL 3.6378/3.5975 and C4 PPL 5.7031/5.6790 before/after correction.
Both paths use FP16 model tensors; these are accuracy evaluations rather than
timing comparisons with the smaller FP32 simulation models. The 70B panel in
Table~\ref{tab:main} reports the completed language-modeling, eight-task
zero-shot, and MMLU evaluations. Gemma-3-1B/4B likewise use complete paired
evaluations for every metric reported in the paper.

\section{Downstream and generation evaluation}
\label{app:accuracy}
The eight-task suite includes PIQA for physical commonsense, SIQA for social
reasoning, ARC-Easy/Challenge and OpenBookQA for science questions, and
HellaSwag, WinoGrande, and BoolQ for contextual reasoning and reading
comprehension~\citep{bisk2020piqa,sap2019socialiqa,clark2018arc,mihaylov2018openbookqa,zellers2019hellaswag,sakaguchi2020winogrande,clark2019boolq}.
The eight-task mean is calculated before rounding the task scores. Original
and incremental task runs share raw PIQA scores and repair-layer counts as an
overlap check. All 57 MMLU subjects are required; the question counts sum to
14,042. LAMBADA~\citep{paperno2016lambada} is reported separately: on
Llama-3-8B, 71.92 for SERQ and 74.11 for attention repair.

The generation study uses paired SERQ/ResCom runs on two Llama and two Gemma
backbones. Each LongBench task has all 200/200/500 examples respectively, with paired
prompt hashes and answer lists matching exactly. The newly integrated Gemma results were independently reaggregated from
7,200 saved per-example LongBench metrics before rounding. Each GSM8K run covers
1,319 questions. Gemma-3-1B downstream measurements cover the complete eight
tasks and all 57 MMLU subjects (14,042 questions).
The Gemma-3-4B text-only extension is evaluated from its native multimodal
checkpoint only after an output-parity check of the extracted text model.

GSM8K uses five shots and both strict exact-match and flexible-extraction
filters from the evaluation harness. The completed Llama-3-8B W4A4 runs give
$316/1319$ and $477/1319$ strict matches. Their individual standard errors
are 1.18 and 1.32 percentage points. The exact difference of the unrounded
rates is 12.21 points. LongBench uses its original benchmark, not LongBench-v2:
QMSum and SAMSum use ROUGE-L F1 and RepoBench-P uses its official first-line
code-similarity metric. The extension uses official prompts, greedy generation,
200/200/500 examples respectively, and a recorded 7,500-token input budget
with middle truncation. The generation limits are 512/128/64 tokens.
This input budget is an explicit implementation choice; SERQ's table alone
does not specify every prompt/truncation detail needed for an exact match.

\begin{table}[!htbp]
\centering\scriptsize
\setlength{\tabcolsep}{2.8pt}
\caption{Zero-shot task details for the principal W4A4 comparisons in
Table~\ref{tab:main}. $\dagger$ marks values published by
SERQ~\citep{park2026serq}; the unmarked SERQ and ResCom rows are our paired
checkpoint evaluations under the common harness. Averages are formed before
rounding.}
\label{tab:accuracy}
\begin{tabular}{llrrrrrrrrr}
\toprule
Backbone & Method & ARC-c & ARC-e & BoolQ & HellaS. & OBQA & PIQA & SIQA & WinoG. & Avg.\\
\midrule
\multirow{7}{*}{LLaMA-2 7B}
 & Full precision$^\dagger$ & 46.25 & 74.49 & 77.74 & 76.00 & 44.20 & 79.05 & 46.11 & 68.90 & 64.09\\
 & SmoothQuant$^\dagger$ & 27.82 & 25.97 & 44.16 & 26.39 & 27.20 & 49.95 & 34.49 & 47.51 & 35.44\\
 & SmoothQuant, g128$^\dagger$ & 39.85 & 65.40 & 68.59 & 67.99 & 37.80 & 74.48 & 41.45 & 61.64 & 57.15\\
 & QuaRot$^\dagger$ & 43.00 & 71.59 & 72.84 & 72.30 & 41.60 & 77.20 & 33.06 & 64.64 & 59.53\\
 & SpinQuant$^\dagger$ & 43.60 & 71.30 & 73.80 & 73.20 & 40.40 & 76.00 & 44.10 & 65.40 & 61.00\\
 & SERQ & 43.26 & 71.97 & 75.47 & 74.05 & 44.00 & 77.91 & 45.04 & 68.43 & 62.52\\
 & ResCom & 42.06 & 71.30 & 76.30 & 74.31 & 43.40 & 77.31 & 45.19 & 68.51 & 62.30\\
\midrule
\multirow{7}{*}{LLaMA-2 13B}
 & Full precision$^\dagger$ & 49.23 & 77.53 & 80.61 & 79.38 & 45.20 & 80.52 & 47.39 & 72.38 & 66.53\\
 & SmoothQuant$^\dagger$ & 27.73 & 26.81 & 38.47 & 25.72 & 25.80 & 49.24 & 34.70 & 47.75 & 34.53\\
 & SmoothQuant, g128$^\dagger$ & 44.45 & 70.50 & 75.23 & 72.52 & 39.00 & 76.93 & 43.19 & 68.43 & 61.28\\
 & QuaRot$^\dagger$ & 46.25 & 73.70 & 78.47 & 76.29 & 43.00 & 78.89 & 33.32 & 70.48 & 62.55\\
 & SpinQuant$^\dagger$ & 49.00 & 76.30 & 78.20 & 77.10 & 42.80 & 79.30 & 46.30 & 69.50 & 64.80\\
 & SERQ & 49.40 & 76.14 & 80.95 & 78.62 & 46.00 & 79.49 & 46.93 & 71.11 & 66.08\\
 & ResCom & 49.66 & 76.22 & 81.44 & 78.41 & 45.20 & 79.65 & 46.72 & 72.45 & 66.22\\
\midrule
\multirow{7}{*}{Llama-3-8B}
 & Full precision$^\dagger$ & 52.99 & 77.69 & 81.07 & 79.20 & 45.00 & 80.74 & 47.08 & 73.48 & 67.16\\
 & SmoothQuant$^\dagger$ & 26.02 & 24.20 & 50.61 & 26.65 & 28.20 & 51.69 & 33.78 & 49.57 & 36.34\\
 & SmoothQuant, g128$^\dagger$ & 30.38 & 46.59 & 61.83 & 55.61 & 33.60 & 64.20 & 39.92 & 59.59 & 48.97\\
 & QuaRot$^\dagger$ & 43.26 & 69.57 & 70.49 & 72.22 & 42.80 & 77.04 & 32.96 & 64.64 & 59.12\\
 & SpinQuant$^\dagger$ & 46.90 & 72.00 & 73.40 & 71.90 & 42.40 & 75.20 & 44.40 & 67.70 & 61.75\\
 & SERQ & 45.22 & 71.30 & 78.23 & 73.93 & 41.80 & 76.88 & 45.24 & 71.27 & 62.98\\
 & ResCom & 49.32 & 74.45 & 80.00 & 76.35 & 44.60 & 78.02 & 46.88 & 69.53 & 64.89\\
\bottomrule
\end{tabular}
\end{table}

\begin{table}[!htbp]
\centering\scriptsize
\setlength{\tabcolsep}{3pt}
\caption{Paired W4A4 zero-shot task details for the additional backbones in
Tables~\ref{tab:main} and~\ref{tab:additional}. Both rows in each pair are
our evaluations of the same checkpoint before and after repair.}
\label{tab:accuracy_additional}
\begin{tabular}{llrrrrrrrrr}
\toprule
Backbone & Method & ARC-c & ARC-e & BoolQ & HellaS. & OBQA & PIQA & SIQA & WinoG. & Avg.\\
\midrule
Llama-3.2-1B & SERQ & 34.30 & 53.32 & 60.06 & 58.35 & 34.60 & 69.53 & 41.25 & 56.04 & 50.93\\
 & ResCom & 33.11 & 55.22 & 59.45 & 58.84 & 35.20 & 69.59 & 40.84 & 55.88 & 51.02\\
\addlinespace[2pt]
Llama-3.2-3B & SERQ & 40.10 & 67.13 & 70.70 & 69.29 & 37.40 & 75.63 & 45.45 & 66.14 & 58.98\\
 & ResCom & 42.32 & 67.42 & 71.22 & 70.78 & 40.20 & 76.77 & 45.91 & 66.46 & 60.14\\
\addlinespace[2pt]
Mistral-7B-v0.3 & SERQ & 51.71 & 78.96 & 80.92 & 78.39 & 46.60 & 80.85 & 46.06 & 73.32 & 67.10\\
 & ResCom & 52.65 & 78.96 & 80.83 & 79.36 & 45.40 & 81.12 & 46.32 & 72.45 & 67.13\\
\bottomrule
\end{tabular}
\end{table}

\begin{table}[ht]
\centering\small
\setlength{\tabcolsep}{4pt}
\caption{Additional-family W4A8 results. Each ResCom row uses the paired SERQ
checkpoint and rank-32 attention repair. All task pairs are complete; bold compares each paired result.}
\label{tab:additional_a8}
\begin{tabular}{lllrrrr}
\toprule
Model & Bits & Method & WikiText-2 & C4 & 0-shot & MMLU\\
\midrule
Gemma-3-1B-PT &  W4A8  &  SERQ  &  \textbf{15.026}  &  \textbf{19.104}  &  \textbf{55.81}  &  \textbf{25.13}  \\
  &    &  ResCom  &  15.058  &  19.138  &  55.55  &  25.12  \\
\midrule
Gemma-3-4B-PT &  W4A8  &  SERQ  &  \textbf{10.753}  &  \textbf{14.901}  &  65.20  &  \textbf{55.73} \\
 &    &  ResCom  &  11.128  &  15.326  &  \textbf{65.34}  &  55.52 \\
\bottomrule
\end{tabular}
\end{table}

\begin{table}[ht]
\centering\scriptsize
\setlength{\tabcolsep}{3.8pt}
\caption{W4A8 generation results.
$\dagger$: published comparator values reported by SERQ~\citep{park2026serq}; SERQ rows are
our paired checkpoint evaluations. Metrics and prompt settings match the W4A4 study in
Table~\ref{tab:generation}. }
\label{tab:generation_a8}
\begin{tabular}{lllccccc}
\toprule
& & & \multicolumn{2}{c}{GSM8K 5-shot $\uparrow$} & \multicolumn{3}{c}{LongBench $\uparrow$}\\
\cmidrule(lr){4-5}\cmidrule(lr){6-8}
Model & Bits & Method & Flexible & Strict & QMSum & SAMSum & RepoBench-P\\
\midrule
\multirow{4}{*}{Llama-3-8B}
& W4A8 & LLM.int4()$^\dagger$ & 7.88 & 4.62 & 13.29 & 34.20 & 44.93\\
& & L$^2$QER$^\dagger$ & 36.24 & 35.33 & 7.28 & 25.81 & 47.03\\
\cmidrule(l){2-8}
& W4A8 & SERQ & \textbf{43.52} & \textbf{43.14} & 23.28 & 44.20 & \textbf{65.98} \\
\rowcolor{gray!12}
&  & ResCom & 43.37 & 42.91 & \textbf{23.53} & \textbf{44.45} & 65.90 \\
\midrule
\multirow{4}{*}{Llama-3.2-3B}
& W4A8 & LLM.int4()$^\dagger$ & 6.76 & 6.44 & 15.83 & 36.42 & 35.24\\
& & L$^2$QER$^\dagger$ & 18.95 & 18.42 & 21.53 & 40.63 & 60.18\\
\cmidrule(l){2-8}
& W4A8 & SERQ & \textbf{20.17} & \textbf{20.09} & \textbf{23.15} & \textbf{43.95} & 64.77 \\
\rowcolor{gray!12}
&  & ResCom & 19.86 & 19.79 & 22.89 & 43.81 & \textbf{64.94} \\
\midrule
Gemma-3-1B-PT & W4A8 & SERQ & \textbf{2.27} & \textbf{1.74} & 17.94 & 7.16 & \textbf{51.10} \\
\rowcolor{gray!12}
 & & ResCom & 2.12 & 1.59 & \textbf{18.14} & \textbf{7.25} & 50.96 \\
\midrule
Gemma-3-4B-PT & W4A8 & SERQ & \textbf{28.81} & \textbf{28.58} & \textbf{21.62} & \textbf{8.54} & \textbf{65.34} \\
\rowcolor{gray!12}
 & & ResCom & 27.67 & 26.91 & 21.60 & 6.99 & 65.20 \\
\bottomrule
\end{tabular}
\end{table}

\FloatBarrier

\section{Implementation and end-to-end measurements}
\label{app:deploy}

\paragraph{Direct end-to-end comparison.}
Figure~\ref{fig:efficiency_latency} uses direct runs of six Llama-3-8B
configurations on one B200. Every timed model first completes the same 141
WikiText-2 windows. FP16 uses FP16 compute; the quantization-simulation models
use FP32 with TF32 disabled. L$^2$QER retains both factor products and their
intermediate/output quantizers at rank 64 on 224 linear layers. QuaRot uses its
Hadamard rotation, GPTQ, and INT4 simulation. ResCom adds rank-32 factors on
128 attention projections to the SERQ checkpoint. Table~\ref{tab:artifact_efficiency}
reports the underlying measurements.

\begin{table}[!htbp]
\centering\small
\caption{Numerical breakdown for Figure~\ref{fig:efficiency_latency}. B200,
Llama-3-8B, batch one, 4K prompt; milliseconds.}
\label{tab:artifact_efficiency}
\begin{tabular}{lrr}
\toprule
Method & Prefill & Time per output token\\
\midrule
FP16 & 209.07 & 15.89\\
L$^2$QER-A8 & 2600.15 & 414.33\\
L$^2$QER-A4 & 2601.00 & 416.03\\
QuaRot & 1470.75 & 78.70\\
SERQ & 1342.32 & 76.79\\
ResCom & 1372.46 & 86.11\\
\bottomrule
\end{tabular}
\end{table}

\begin{figure}[!htbp]
\centering
\includegraphics[width=\textwidth]{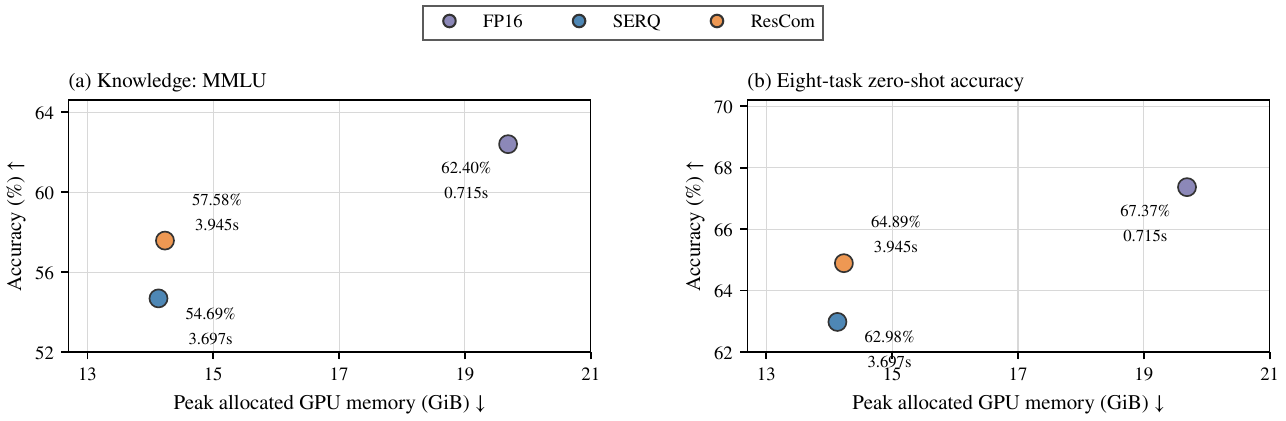}
\caption{Full memory--accuracy views for MMLU and the eight-task mean, using
the measurements summarized in Figure~\ref{fig:efficiency_memory}. Each
point is labeled with its measured end-to-end latency.}
\label{fig:memory_full}
\end{figure}

\begin{figure}[!htbp]
\centering
\includegraphics[width=\textwidth]{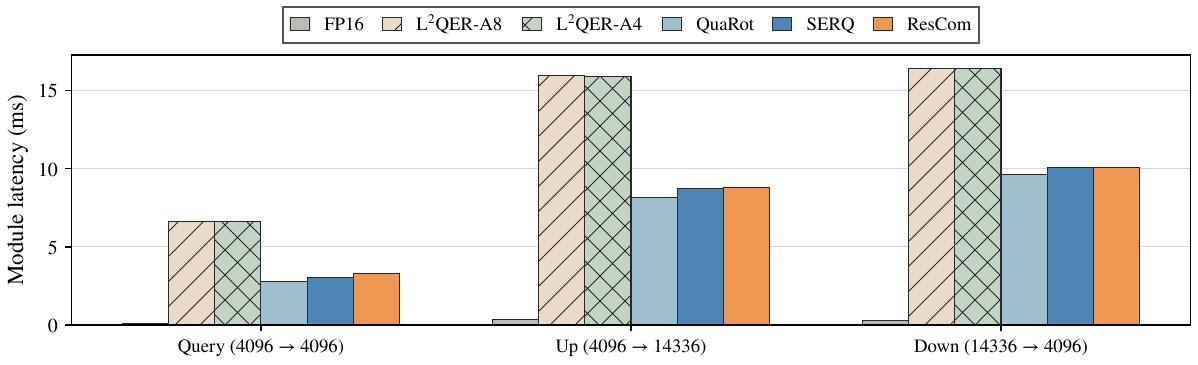}
\caption{Per-module latency for the layer-0 query, up, and down projections of
the six evaluated Llama-3-8B configurations on B200, batch one with 4096 input
tokens. The measurements use each checkpoint's actual inputs and weights and
complement the end-to-end results in Figure~\ref{fig:efficiency_latency}.}
\label{fig:completed_modules}
\end{figure}

\begin{figure}[!htbp]
\centering
\includegraphics[width=\textwidth]{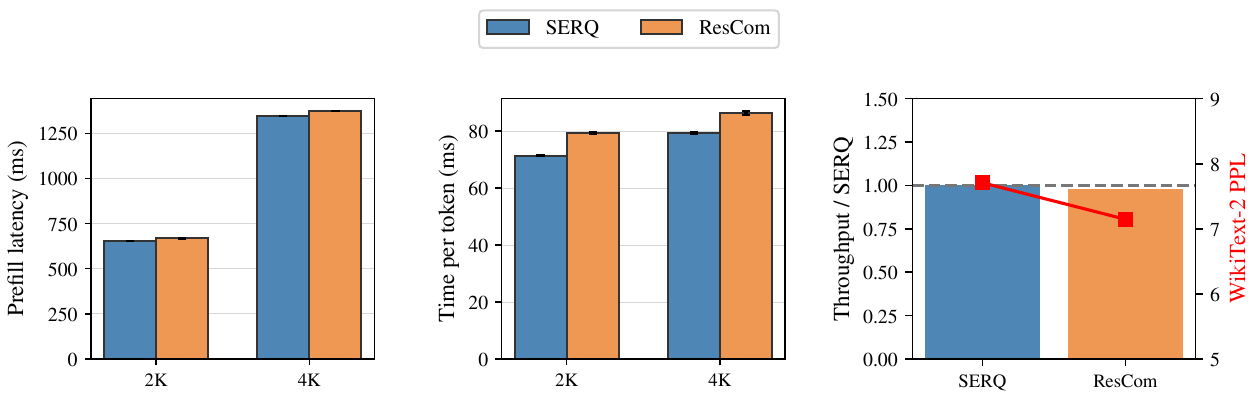}
\caption{Paired SERQ and ResCom latency across prompt lengths on B200, batch
one, with the corresponding WikiText-2 perplexity. Means and sample standard
deviations use three interleaved groups. The shaded difference is the measured
increment from the fitted rank-32 attention branch.}
\label{fig:paired_latency_detail}
\end{figure}

\paragraph{Compact-storage protocol.}
Figure~\ref{fig:efficiency_memory} uses a separate fresh-process measurement
with 4096 input tokens and 32 generated tokens. Peak allocated memory includes
codebooks, factors, KV cache, and workspace. Four-bit weight indices and FP32
per-group codebooks preserve the evaluated weight values; one layer is
dequantized at a time for FP32 GEMM. Lossless weight round-trip and full-model
logit-parity checks connect the packed representation to the evaluated
checkpoint. Three repetitions give peak memory of
19.6862/14.1303/14.2330\,GiB and mean wall time of
0.7150/3.6972/3.9455\,s for FP16/SERQ/ResCom. Timing uses three groups, three
warmups, and five prefill calls per group; cached decoding recreates the prefix
cache and generates exactly 32 tokens.

\paragraph{Fused residual frontend.}
The kernel in Table~\ref{tab:fused_frontend} uses the saved ResCom factors and
captured validation activations. It fuses activation quantization with the
first residual projection, emits two INT4 codes per byte with FP32 scale and
zero point, and reduces group projections before the second factor product.
All twelve shape/token tests reproduce the reference dequantized activations
and keep relative correction error below $2\times10^{-5}$. Buffers are allocated
before timing; each group uses ten warmups and fifty calls.
\begin{table}[!htbp]
\centering\small
\renewcommand{\arraystretch}{0.94}
\caption{Fused residual frontend with saved rank-32 factors on B200. Inputs
are captured Llama-3-8B layer-0 activations; latency is in microseconds and
averaged over three interleaved groups. The fused path also emits packed INT4
codes; main-path GEMM is outside this measurement.}
\label{tab:fused_frontend}
\begin{tabular}{lrrrr}
\toprule
Projection & Tokens & Eager & Fused & Ratio\\
\midrule
q\_proj & 1 & 163.8 & 48.2 & 3.40\\
q\_proj & 16 & 167.9 & 52.1 & 3.22\\
q\_proj & 128 & 167.3 & 52.3 & 3.20\\
q\_proj & 2048 & 288.2 & 164.6 & 1.75\\
k\_proj & 1 & 160.2 & 46.8 & 3.42\\
k\_proj & 16 & 161.6 & 47.2 & 3.42\\
k\_proj & 128 & 168.8 & 51.3 & 3.29\\
k\_proj & 2048 & 282.3 & 155.6 & 1.81\\
o\_proj & 1 & 159.4 & 46.9 & 3.40\\
o\_proj & 16 & 167.9 & 51.2 & 3.28\\
o\_proj & 128 & 168.0 & 51.2 & 3.28\\
o\_proj & 2048 & 289.4 & 164.6 & 1.76\\
\bottomrule
\end{tabular}
\end{table}

\FloatBarrier

\begin{figure}[!htbp]
\centering
\includegraphics[width=\textwidth]{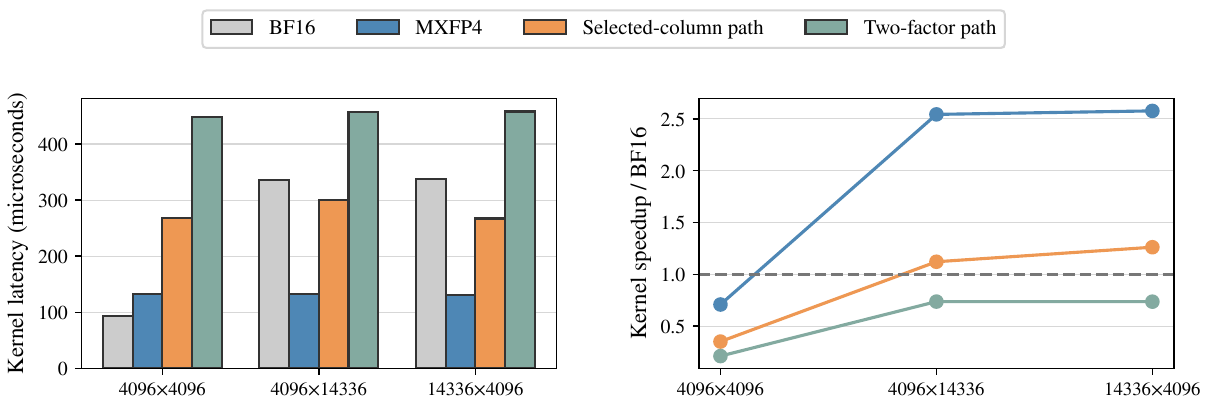}
\caption{Packed kernel-component measurements for the selected-column path
and the two-factor path. The former performs one thin product and the latter
performs two. Shapes are $(K,N)$, with input width $K$ and output width $N$;
BF16 and MXFP4 are the executed dtypes.}
\label{fig:prototype_gemm}
\end{figure}
\FloatBarrier

\paragraph{Native MXFP4 execution.}
A second implementation freshly quantizes Llama-3-8B and refits rank-32 ResCom
factors for native MXFP4. All 224 linear projections invoke native MXFP4 matrix
multiplication. Each fresh process runs batch one with the indicated prompt and
exactly 32 generated tokens. Three repetitions include quantization, padding,
residual computation, factor products, reconstruction, and KV-cache updates;
checkpoint loading and fitting are excluded. Table~\ref{tab:native_mxfp4}
shows the resulting accuracy, end-to-end time, and peak allocation. This
implementation verifies that the correction operates with a native low-bit
main path; the current wrapper is a functional implementation rather than an
optimized serving kernel.

\begin{table}[!htbp]
\centering\small
\caption{Llama-3-8B native MXFP4 execution on B200, batch one and 32 generated
tokens. PPL is full WikiText-2 and shared across prompt lengths; time is the
mean end-to-end seconds of three repetitions.}
\label{tab:native_mxfp4}
\begin{tabular}{llrrr}
\toprule
Prompt & Path & PPL $\downarrow$ & Time (s) $\downarrow$ & Peak (GiB) $\downarrow$\\
\midrule
128 & BF16 & 6.138 & 0.531 & 15.013\\
& MXFP4 & 7.865 & 6.826 & 8.172\\
& MXFP4 ResCom & 7.733 & 7.017 & 8.275\\
2048 & BF16 & 6.138 & 0.592 & 16.568\\
& MXFP4 & 7.865 & 6.854 & 9.588\\
& MXFP4 ResCom & 7.733 & 7.186 & 9.690\\
4096 & BF16 & 6.138 & 0.788 & 20.686\\
& MXFP4 & 7.865 & 7.302 & 11.097\\
& MXFP4 ResCom & 7.733 & 7.571 & 11.200\\
\bottomrule
\end{tabular}
\end{table}

\FloatBarrier
\section{Calibration time}
\label{app:deploy_ablation}
Calibration collects moments and solves the rank-constrained regression once;
task evaluation is excluded. Table~\ref{tab:calibtime} reports complete
rank-32 attention-repair jobs. The independently timed Llama-3.2-3B run takes
184.4 seconds and reproduces the reported 8.9638 PPL, consistent with the
185-second rank-sweep job.

\begin{table}[!htbp]
\centering\small
\caption{One-time rank-32 attention-repair calibration on B200. Times include
moment collection, fitting, and diagnostics; original PTQ and task evaluation
are excluded.}
\label{tab:calibtime}
\begin{tabular}{lrrrr}
\toprule
Model & Llama-3.2 1B & Llama-3.2 3B & LLaMA-2 7B & LLaMA-2 13B\\
\midrule
Fit time (s) & 58 & 185 & 544 & 1353\\
\bottomrule
\end{tabular}
\end{table}

All rows repair the query, key, value, and output projections. Fit time includes
moment collection, factorization, and the diagnostic spectrum used in the
paper. The corresponding control/repair PPL pairs are 9.1769/8.9638 for
Llama-3.2-3B, 5.9389/5.8475 for LLaMA-2-7B, and 5.2274/5.1484 for
LLaMA-2-13B.

\end{document}